\documentclass{article} 
\usepackage{iclr2024_conference,times}
\usepackage{amsthm}

\usepackage{amsmath,amsfonts,bm}

\def\eqref#1{equation~\ref{#1}}

\def\1{\bm{1}}

\def\vtheta{{\bm{\theta}}}
\def\vpsi{{\bm{\psi}}}

\DeclareMathAlphabet{\mathsfit}{\encodingdefault}{\sfdefault}{m}{sl}
\SetMathAlphabet{\mathsfit}{bold}{\encodingdefault}{\sfdefault}{bx}{n}

\usepackage{hyperref}
\usepackage{url}
\usepackage{graphicx}
\usepackage{subfig}
\usepackage{mathrsfs}
\usepackage{algorithm2e} 
\RestyleAlgo{ruled}
\usepackage{thmtools} 
\usepackage{thm-restate}
\usepackage{wrapfig}

\usepackage[toc,page,header]{appendix}
\usepackage{minitoc}

\SetKwComment{Comment}{/* }{ */}

\title{Do Transformer World Models Give \\ Better Policy Gradients?}

\author{%
Michel Ma\thanks{Equal contribution.},\,  
Tianwei Ni,\, 
Clement Gehring,\, 
Pierluca D'Oro\footnotemark[1],\, 
Pierre-Luc Bacon 
\\
  Mila, Université de Montréal
  \\
  {\footnotesize\texttt{\{michel.ma, tianwei.ni, clement.gehring\}@mila.quebec}}\\
  {\footnotesize\texttt{\{pierluca.doro, pierre-luc.bacon\}@mila.quebec}}
}

\renewcommand{\cite}{\citep}

\newtheorem{remark}{Remark}

\iclrfinalcopy 
\begin{document}

\doparttoc 
\faketableofcontents 

\maketitle

\begin{abstract}
A natural approach for reinforcement learning
is to predict future rewards by unrolling a 
neural network world model, and to backpropagate through the resulting computational graph to learn a policy.
However, this method often becomes impractical for long horizons since typical world models induce hard-to-optimize loss landscapes. 
Transformers are known to efficiently propagate gradients over long horizons: could they be the solution to this problem?
Surprisingly, we show that commonly-used transformer world models 
produce \emph{circuitous gradient paths}, which can be detrimental to long-range policy gradients.
To tackle this challenge, we propose a class of world models called Actions World Models (AWMs), designed to provide more direct routes for gradient propagation. We integrate such AWMs into a policy gradient framework that underscores the relationship between network architectures and the policy gradient updates they inherently represent.
We demonstrate that AWMs can generate optimization landscapes that are easier to navigate even when compared to those from the simulator itself. This
property allows transformer AWMs to produce better policies than competitive baselines 
in realistic long-horizon tasks.
\end{abstract}

\section{Introduction}

Given a class of parameterized policies, policy optimization methods aim to find parameters $\vtheta$ that maximize a performance criterion $J$ -- typically the expected sum of rewards over the episodes of experiences generated by an agent in an environment. 
When the environment dynamics are known and differentiable, the gradient of the objective $J$ can be directly evaluated through automatic differentiation and used 
to optimize $J$ by gradient ascent.
If the dynamics are unknown, a differentiable world model can be trained on interaction data, whose predicted states will be backpropagated through.
This model-based approach to reinforcement learning~(RL) has been explored early on in the field with~\citet{werbos1974beyond, schmidhuber1990making, miller1995neural} and led to contemporary work in deep RL such as~\citet{hafner2022mastering,hafner2023mastering}. 
While backpropagation through time~\cite{werbos1990backpropagation} using a world model is conceptually appealing and readily compatible with deep learning practice, most attempts to scale it up have been unsuccessful beyond a few dozen steps~\cite{hafner2022mastering, ghugare2023simplifying}.
Indeed, even in the presence of perfectly accurate differentiable simulators, policy gradients obtained by backpropagation through time can quickly become unstable over long horizons~\cite{SuhSZT22, metz2021gradients}.


\begin{figure*}[t]
    \centering
    \subfloat[Markovian world model.]
    {{\includegraphics[width=0.33\linewidth]{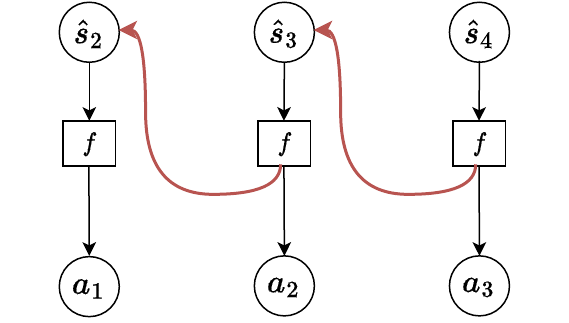}\label{fig:markovian-vis}}}
    \subfloat[History World Model.]
    {{\includegraphics[width=0.33\linewidth]{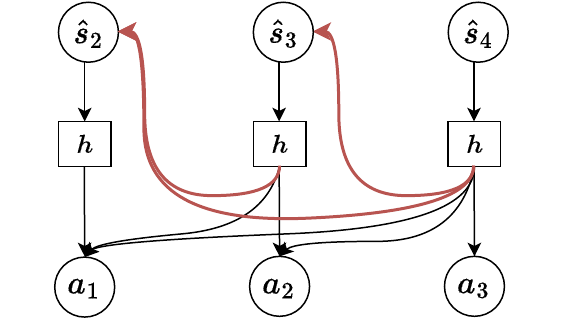}\label{fig:hwm-vis}}}
    \subfloat[Actions World Model.]
    {{\includegraphics[width=0.33\linewidth]{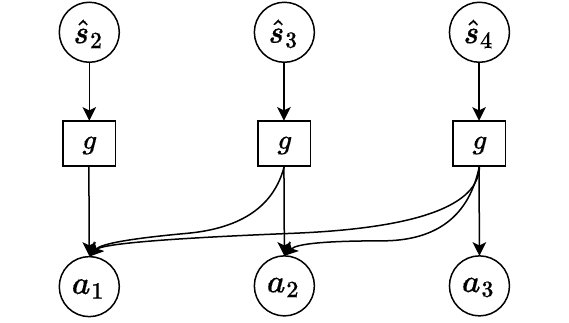}\label{fig:awm-vis}}}
    \caption{\footnotesize \textbf{Diagram illustrating gradient flows through different world model types from states to actions.} Circuitous (longer than necessary) gradient paths go through connections highlighted in red. An Actions World Model has no circuitous gradient paths, allowing gradients to directly flow from states to actions through a single application of a world model.}
    \label{fig:diagram}
    \vspace{-1em}
\end{figure*}
This problem of stability in differentiating through world models over long horizons echoes similar challenges observed in sequence modeling. The latter field has experienced numerous advances in deep network architectures specifically aimed at alleviating this challenge.
As a notable example, attention-based models such as transformers were designed to directly propagate gradients from any output to any input arbitrarily far in the past, within their context length~\cite{schmidhuber1992learning,Vaswani2017}.
It is, therefore, natural to ponder whether using a transformer as a world model and differentiating through it could overcome the difficulties encountered in model-based policy optimization.
Or, in other words, to ask the question: \emph{do transformer world models give better policy gradients?} In this paper, we provide a new answer to this question: surprisingly, transformers do not necessarily confer the same benefit for policy optimization when used as world models conditioned on the full history of states and actions. Instead, conditioning solely on a sequence of actions tends to yield better-behaved policy gradients. This, in turn, leads to improved temporal credit assignment and policy performance.

We first show that backpropagating naively through a transformer predicting the next state conditioned on the full history 
~\cite{micheli2022transformers, chen2022transdreamer, robine2023transformer} does not necessarily lead to better policy gradients in long-horizon tasks.
We demonstrate, conceptually and theoretically, that this phenomenon can be attributed to the existence of \emph{circuitous gradient paths}: long state-to-state gradient paths generated during the autoregressive unrolling of the model~(see Figure~\ref{fig:diagram}). 

To avoid the creation of those longer-than-necessary gradient paths, 
we propose an alternative class of world models, which we call \emph{Actions World Models}~(AWMs), of the form $ {\hat{s}_{t+1} = g(s_1, a_1, \dots, a_{t-1}, a_t ) } $ conditioned solely on the history of past actions and an initial state. 
An AWM is forced to create its own internal dynamics, potentially very different from the one of the environment, while being able to predict future states. This inductive bias constrains the temporal accumulation of gradients to only happen within the neural network of the AWM,
fully harnessing its gradient properties. 

We show theoretically that an AWM directly inherits the gradient propagation properties of the underlying network architecture.
Thus, AWMs yield policy gradients subject to bounds similar to the ones previously derived in deep sequence models~\cite{PascanuMB13, KergKGGBL20}. 
Additionally, our theoretical framework provides a unified perspective on many backpropagation-based policy optimization methods by exposing the relationship between the network architecture and the policy gradient updates that they induce.
For instance, it allows us to cast backpropagation through unrolled Markovian dynamics~\cite{HeessWSLET15} as a specific AWM instantiated with a vanilla recurrent neural network. Our framework paves the way for new algorithmic developments by fostering novel synergies between the study of network architectures in deep learning and the policy optimization updates they implicitly express.

Finally, through a series of experiments, we showcase the remarkable empirical properties of backpropagation-based policy optimization with AWMs.  
We demonstrate in illustrative domains with chaotic or non-differentiable dynamics that AWMs can give better policy gradients than the simulator itself, by forcing the policy optimization landscape to be easy to optimize. This leads to successful policy optimization via backpropagation even when the same procedure would fail with the true underlying model. 
On a testbed of realistic long-horizon tasks~\cite{howe_myriad}, we benchmark policy optimization with transformer AWMs against other competitive model-free and model-based methods, showing superior performance.

\section{Background}
\label{sec:background}

\textbf{Problem definition}~ We are interested in deterministic\footnote{
See \citet{laidlaw2023bridging} for a discussion on the importance of the study of deterministic MDPs.} discrete-time finite-horizon Markov Decision Processes~(MDPs)~\citep{fairbank2014value}, defined as $\mathcal M = (\mathcal S, \mathcal A, f, r, H, s_1)$, where $\mathcal S \subseteq \mathbb{R}^n$ is the state space, $\mathcal A \subseteq \mathbb{R}^m$ is the action space, $f: \mathcal S \times \mathcal A \to \mathcal S$ is the differentiable transition dynamics, $r: \mathcal S \to \mathbb R$ is the known differentiable reward function, $H$ is the horizon and $s_1 \in \mathcal S$ is the initial state.
The behavior of an agent in the environment is described by a policy $\pi_{\vtheta}: \mathcal S \to \mathcal A$, belonging to a space of parameterized differentiable deterministic\footnote{This assumption can be relaxed to a stochastic policy using the reparameterization trick~\citep{KingmaW13}.} stationary Markov policies $\Pi = \{ \pi_{\vtheta} : \vtheta \in \mathbb{R}^d \}$.
For clarity, we focus on stationary rewards and dynamics, but our results could readily be adapted to the non-stationary case.

\textbf{Backpropagation-based Policy Optimization} The goal of the agent is to maximize the cumulative episodic rewards. In the rest of the paper, we consider the following problem:
\vspace{-0.4em}
\begin{equation} 
\label{eq:mdpobj}
\begin{gathered}
  \begin{aligned}
    \text{maximize} \quad & J^f(\vtheta; H) := \sum_{t=1}^H r(s_t),  & \text{subject to} \quad & s_{t+1} = f(s_{t}, a_{t}),\, a_k = \pi_\vtheta (\text{sg}[s_k]) \enspace .
  \end{aligned}
\end{gathered}
\end{equation}
where $\mathrm{sg}[\cdot]$ is the stop-gradient operator. 
Model-based policy gradient methods solve this problem by unrolling (simulating) the dynamics forward and computing the gradient of the corresponding unconstrained problem. The presence of the stop gradient operator in our formulation aligns with the contemporary practice~\citep{hafner2022mastering,hafner2023mastering,ghugare2023simplifying}. Its inclusion in our formulation does not impact 
 performance (see Appendix \ref{app:stopgrad-ablation} for an ablation) and simplifies our theoretical analysis.

\begin{wrapfigure}{R}{0.5\textwidth}
\vspace{-0.5cm}
\begin{algorithm}[H]
   \small
   \caption{Backpropagation-based Policy Optimization (BPO)}
   \label{alg:alg}
   \textbf{Input:} Initial buffer $\mathcal B$, initial policy parameters $\vtheta$, initial model parameters $\vpsi$, learning rates $\{ \alpha_{\vtheta}, \alpha_{\vpsi} \}$, world model class~$\mathcal W$.
   \begin{algorithmic}[1] 
       \WHILE{not exceeding training steps}
       \STATE Collect an episode with $\pi_{\vtheta}$ and add it to $\mathcal B$
       \FOR{each world model learning step}
       \STATE $\vpsi \gets \vpsi - \alpha_{\vpsi} \nabla_{\vpsi} \ell^{\mathcal W}(\tau; \vpsi),\quad \tau \sim \mathcal B$
       \ENDFOR
       \FOR{each policy learning step}
       \STATE Compute $J^{\mathcal W}(\vtheta; \vpsi)$ by unrolling the world model
       \STATE Compute $\nabla_{\theta} J^{\mathcal W}(\vtheta; \vpsi)$ by backpropagation
       \STATE $\vtheta \gets \vtheta + \alpha_{\vtheta}  \nabla_{\vtheta} J^{\mathcal W}(\vtheta; \vpsi)$
       \ENDFOR
       \ENDWHILE
   \end{algorithmic}
\end{algorithm}
\vspace{-2em}
\end{wrapfigure}
When the true dynamics $f$ are unknown, 
a transition function $\hat{f}_\vpsi(s_t, a_t)$ belonging to a space of parameterized differentiable functions $\mathcal{F} = \{ f_\vpsi : \vpsi \in \mathbb{R}^{d_\psi} \}$ must be learned to perform policy optimization. 
Given transitions $s_t, a_t, s_{t+1}$ sampled from the environment, an approximate Markovian world model $\hat{f}_\vpsi$ can be learned with the loss ${\ell^{\hat f}(\tau;\vpsi) = \sum_{t=1}^{H-1} \| s_{t+1}^\tau - \hat{f}_\vpsi(s_t^\tau, a_t^\tau) \| ^2}$,
where ${\tau = (s_1^\tau,a_1^\tau,\dots,s_{H}^\tau)}$ is a trajectory. For the remainder of this paper, consider $\| \bullet\|$ to denote the $L^2$ norm. 
Given a learned model, policy and initial state, the corresponding return $J^{\hat f}(\vtheta; \vpsi)$ can be computed through rollouts and the gradient $\nabla_{\vtheta} J^{\hat{f}} (\vtheta; \vpsi)$ can be obtained  by backpropagation.
We call this basic algorithm, outlined in Algorithm~\ref{alg:alg}, \emph{Backpropagation-based Policy Optimization}~(BPO).

\section{Policy Gradient Computation with History World Models}

When running Algorithm~\ref{alg:alg}, the performance of the resulting policy depends on the quality of the gradient approximation.
Effective long-term credit assignment capabilities are required for computing the gradient $\nabla_{a_t} r(s_{t+k})$, as provided by the world model, for potentially very large values of $k$ in challenging control tasks. This means understanding how a reward at a temporally distant time changes in response to variations in a specific action taken at an earlier moment.
Therefore, the long-term credit assignment properties of a model-based policy optimization algorithm are intimately tied to the structure of the unrolled world model.

One fundamental fact about backpropagation is that its success in learning long-term dependencies is intimately linked to the length of the paths involved in the neural network's gradient computation~\citep{pascanu2013construct,zhang2016architectural}.
Indeed, the success of transformers is often attributed to this phenomenon: ``one key factor affecting the ability to learn such dependencies is the length of the paths forward and backward signals have to traverse in the network''~\citep{Vaswani2017}. 
When trained for sequence modeling tasks, transformers stand out compared to other architectures because the length of these paths does not linearly increase with the sequence length.

Recent model-based RL approaches have introduced the use of transformer History World Models~(HWMs), which predict the next state as $\hat{s}_t = h(s_{1:{t-1}}, a_{1:{t-1}})$ based on the full history of states and actions ~\citep{micheli2022transformers,robine2023transformer}. Their use leads to improved performance due to better prediction abilities but, surprisingly, these approaches were not able to take advantage of the gradient propagation properties of the transformer architecture, as exemplified by previous work which only unrolls them for a few steps~\citep{chen2022transdreamer}.

In practice, an HWM is unrolled in this manner to compute the cumulative reward for a given policy:
\vspace{-0.4em}
\begin{equation} 
\label{eq:historyobj}
\begin{gathered}
  \begin{aligned}
    \text{maximize} \quad & J^h(\vtheta; H) := \sum_{t=1}^H r(\hat{s}_t), & \text{subject to} \quad & \hat{s}_{t+1} = h(\hat{s}_{1:t},  a_{1:k}), \, a_k = \pi_\vtheta (\text{sg}[\hat{s}_k]) \enspace .
  \end{aligned}
\end{gathered}
\end{equation}

HWMs are typically learned by employing a prediction objective $\ell^{h}(\tau;\vpsi) = \sum_{t=1}^{H-1} \| s_{t+1}^\tau - h_\vpsi(s_{1:t}^\tau, a_{1:t}^\tau) \| ^2$, similar to the one employed for a Markovian model,
where $s_{1:t}^\tau$ and $a_{1:t}^\tau$ denote subsequences of states and actions from an episode $\tau$.
While history-conditioned dynamics are in principle not necessary in an MDP\footnote{Policy gradients can be computed through equation \ref{eq:historyobj} similarly in partially observable MDPs (POMDP), where states are replaced with observations. The resulting analysis would, therefore, still hold in a POMDP.}, one may hope that the gradient properties of transformers might manifest themselves positively in the policy gradient. It stands to reason that the policy gradients through a transformer may be able to more effectively capture long-term dependencies than a Markovian model would.
Nonetheless, this apparently natural usage of a transformer has unintended consequences for policy optimization. 

Visually, we show in Figure \ref{fig:diagram} the gradient paths induced by unrolling a \emph{state-based} world model, whether Markovian (Figure \ref{fig:markovian-vis}) or history-dependent (Figure \ref{fig:hwm-vis}), from state predictions to actions. The impact of an action on a future state is measured by navigating the paths connecting the two generated by a world model. Some of these are \emph{circuitous gradient paths}: 
these paths do not go from state to action directly but instead include auto-regressive state predictions from the model itself.
Such paths are perilous for long-horizon policy optimization; errors, noise, and gradient behaviors that may be good internally can rapidly degrade outside the model, reducing the possible advantages of backpropagating through transformers. 

In particular, the following theorem shows that even if the gradient of a transformer HWM is bounded, the resulting policy gradient can still grow exponentially. The proof of this and the rest of our results can be found in Appendix~\ref{sec:proofs}.
\begin{restatable}{theorem}{histExplode}
\label{thm:hist-explode}
    Let the gradient norm of $h$ with respect to its inputs be bounded by $L_a$ and $L_s$: $\|\frac{\partial h(\hat{s}_{1:t}, a_{1:t})}{\partial a_k}\| \leq L_a$ and $\| \frac{\partial h(\hat{s}_{1:t}, a_{1:t})}{\partial \hat{s}_i}\| \leq L_s$ for all $s_{1:t}, a_{1:t}, k, i$.  Let $r$ be the $L_r$-Lipschitz reward function from a Markov Decision Process $\mathcal M$, $\Pi_\vtheta$ a parametric space of differentiable deterministic $L_\pi$-policies. Given $\pi_\vtheta \in \Pi_\vtheta$, the norm of the policy gradient $\nabla_\vtheta J^h(\vtheta; H)$ of $\pi_\vtheta$ under a History World Model $h$ grows asymptotically as a function of the horizon $H$ as:
    \begin{align*}
        \| \nabla_\theta J^h (\vtheta; H) \| &=O(HL_r +H^2 L_\pi + H^2 L_a + H^2 L_s^H) = O(L_s^H) \enspace .
    \end{align*}
\end{restatable}

\begin{wrapfigure}[16]{r}{0.4\textwidth}
\vspace{-0.3cm}
    \centering
    \includegraphics[width=0.95\linewidth]{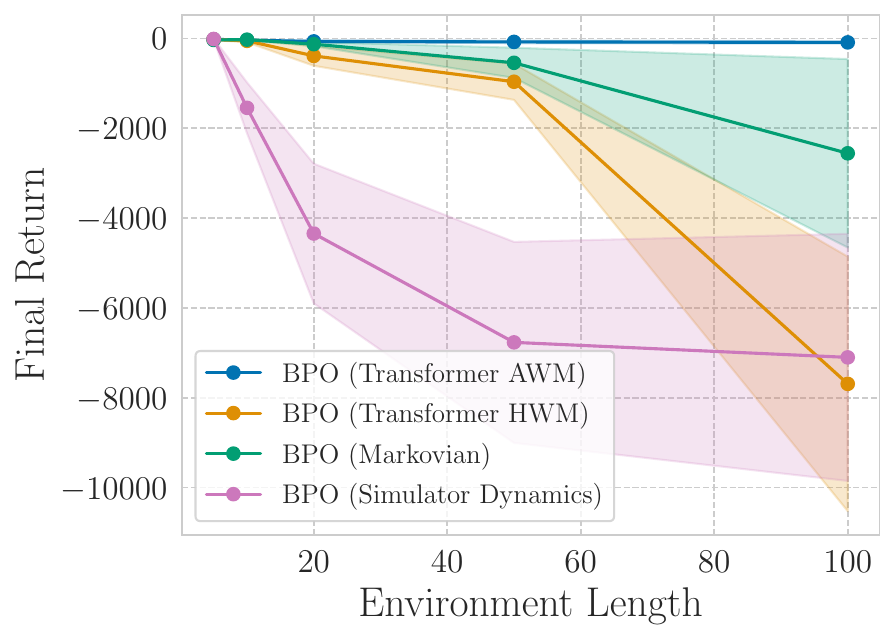}
    \caption{\footnotesize\textbf{Transformer AWMs outperforms all BPO baselines in chaotic environments.} Final performances of BPO with different world models on the double-pendulum environment (10 seeds $\pm$ std).}
    \label{fig:double-pend-rl}
\end{wrapfigure}
The exponential dependency, related to exploding gradients in RNNs, is apparent in Figure~\ref{fig:hwm-vis}, where the longest path from any given state to action is no shorter than the longest path through a Markovian model. The bound is shown to be tight in Appendix \ref{tightness-hist}. In the one-dimensional case, when all gradients are positive scalars, the above bound can be turned into a lower bound.

To illustrate the negative effects of circuitous gradient paths in practice, we conduct experiments in a double-pendulum environment where the agent must move the double-pendulum to a target position. We describe and visualize the task more precisely in Section~\ref{sec:experiments}.
The results in Figure~\ref{fig:double-pend-rl} show that BPO does not lead to successful policies with both Markovian models and HWMs in this notoriously chaotic environment. 
This result demonstrates that state-conditioned transformers do not provide any benefits to policy gradients due to circuitous gradient paths, unlike in the traditional supervised learning setting where they were initially introduced.
In the rest of the paper, we will see how using transformers as a different type of world model allows us not only to overcome the limitations of Markovian models and HWMs, but also to outperform policy optimization performed by differentiating through the dynamics of the simulator, as well as competitive model-free approaches.

\section{Policy Gradient Computation with Actions World Models}
Policy gradients obtained from state-based world models, whether Markovian models or HWMs, are necessarily computed by backpropagating through sequences of states generated by autoregressively unrolling the models. Each state prediction uses as input previously generated states; consequently, the longest path between any action and reward traverses through circuitous gradient paths in the state space, \emph{regardless of the underlying world model architecture}. Through careful examination of Figure \ref{fig:diagram} and Theorem~\ref{thm:hist-explode}, it quickly becomes apparent that the state dependency in $h$ or $f$ is at the root cause of such paths. 

Inspired by this observation, we propose to remove all circuitous gradient paths in the simplest and most direct way: conditioning the world model on actions only. We thus consider a model which predicts a future state given an initial state and a sequence of actions. The resulting model, which we call an \emph{Actions World Model}~(AWM), is of the form $\hat{s}_{t+1} = g(s_1, a_1, a_2, \hdots, a_t)$.
This structure for the world model constrains gradients to pass directly from rewards to actions through the network parameterizing the AWM. Note that previous work showed that actions are theoretically sufficient to predict future states in deterministic MDPs and often sufficient in practice for most stochastic environments~\citep{rezende_causallycorrect}.

Given an AWM, a policy is learned by computing and employing (as in Algorithm~\ref{alg:alg}) the gradient of the objective:
\vspace{-0.4em}
\begin{equation} 
\label{eq:awmobj}
\begin{gathered}
  \begin{aligned}
    \text{maximize} \quad & J^g(\vtheta; H) := \sum_{t=1}^H r(\hat s_t), &
    \text{subject to} \quad &\hat{s}_{t+1}= g(s_1, a_{1:t}), \,a_k = \pi_\vtheta (\mathrm{sg}[\hat{s}_k]) \enspace .
  \end{aligned}
\end{gathered}
\end{equation}
An AWM can be trained in a similar manner to a Markovian model or an HWM, by using the loss ${\ell^{g}(\tau;\vpsi) = \sum_{t=1}^{H-1} \| s_{t+1}^\tau - g_\vpsi(s_1, a_{1:t}^\tau) \|^2}$,
computed likewise using states from an episode $\tau$.

We will now show that, despite the seemingly counter-intuitive removal of state inputs, AWMs naturally connect the theory of policy gradients and sequence modeling with neural networks and provide a framework in which transformers can give high-quality policy gradients.

\subsection{Markovian Models are RNN Actions World Models}
The intuition that policy gradients computed through an unrolled Markovian model prescribed by \eqref{eq:mdpobj} resembles the process of backpropagating through a recurrent neural network has been pointed out in multiple prior works~\cite{HeessWSLET15, metz2021gradients, SuhSZT22, zhang2023modelbased}. Yet, a formal account of this relationship has never been shown. 
Via the framework provided by AWMs, the next proposition shows how recurrent neural networks and policy gradients are connected, offering a first glimpse at the connection between neural network architectures, AWMs and policy gradient computation in MDPs.
\begin{restatable}{proposition}{rnnEquivalence}
Let 
$f \text{-}\texttt{RNN}$ 
be a recurrent network with its recurrent cell being the dynamics $f$ of the MDP $\mathcal M$, and $g_{ f \text{-}\texttt{RNN}}$ denote an AWM instantiated with $f \text{-}\texttt{RNN}$. Then,
\begin{align*}
    \nabla_\vtheta J^{g_{f \text{-}\texttt{RNN}}}(\vtheta; H) = \nabla_\vtheta J^f(\vtheta; H).
\end{align*}
\label{pro:rnn_equivalence}
\end{restatable}
\vspace{-2em}
The above proposition tells us that the policy gradient computed through a Markovian model is, in fact, equivalent to the one computed through an AWM when instantiating $g$ as a recurrent neural network with a specific recurrent cell. Crucially, this not only provides grounding for gradient estimation with AWMs but also solidifies a fundamental fact that will be analyzed in-depth in this section: policy gradient computation by differentiating through unrolled Markovian models can be understood to be fundamentally ill-behaved due to its correspondence to an RNN structure. 

\subsection{Theoretical Properties of Actions World Models}
Through the concept of Actions World Models, we have established a direct connection between deep sequence models and policy gradient computation. We will now exploit this connection to characterize the asymptotic behavior of the policy gradient depending on the underlying neural network architecture employed as an AWM. To do so, we leverage the following theorem.
\begin{restatable}{theorem}{generalBound}
Let $r$ be the $L_r$-Lipschitz reward function from a Markov Decision Process $\mathcal M$, $\Pi_\vtheta$ a parametric space of differentiable deterministic $L_\pi$-policies. Given $\pi_\vtheta \in \Pi_\vtheta$, the norm of the policy gradient $\nabla_\vtheta J^g(\vtheta; H)$ of $\pi_\vtheta$ under an action world model $g$ as a function of the horizon $H$ can upper bounded as:
\label{thm:general-bound}
\begin{align*}
    \left \| \nabla_\vtheta J^g(\vtheta; H) \right \| \leq L_r L_\pi  \sum_{t=1}^H \sum_{k=1}^{t-1} \left\| \frac{\partial g(s_1, a_{1:t-1})}{\partial a_k} \right\|.
\end{align*}
\end{restatable}
This theorem holds for any differentiable AWM $g$, and establishes a worst-case relationship between the Jacobian of the AWM w.r.t. its inputs (i.e., actions) and the policy gradient.
Notably, the result implies that the policy gradient does not explode if the Jacobian of the AWM does not explode. In contrast, this may not be true even for well-behaved HWM.

We will now leverage the generality of the previous result to characterize the asymptotic behavior of the policy gradient computed using different neural network architectures for an AWM. First, let us consider an AWM instantiated with a simple recurrent neural network: 
\begin{flalign} \label{eq:awmrnn}
 \text{(}g_{\texttt{RNN}}\text{)} && x_{t+1} = \sigma(W_x x_{t}) + W_a a_{t} + b ; && & \hat{s}_{t+1} = W_o x_{t+1}, &&
\end{flalign}
 where $\sigma$ is an activation function with gradient norm bounded by $\left \| diag(\sigma'(x))\right \| \leq \frac{1}{\beta}$ for some constant $\beta$.
 Then, the following result holds.
\begin{restatable}{corollary}{rnnBound}
Let $g_{\texttt{RNN}}$ be an Actions World Model instantiated with a recurrent neural network as in Equation~\ref{eq:awmrnn} and $\eta = \| W_x^T\| \frac{1}{\beta}$. The asymptotic behavior of the norm of the policy gradient $\nabla_\vtheta J^{g_{\texttt{RNN}}} (\vtheta;H)$ as a function of the horizon H can be described as:
    \begin{align*}
        \left \|\nabla_{\vtheta} J^{g_\texttt{RNN}} (\vtheta; H) \right \|= O\left ( \eta^H\right).
    \end{align*}
    \label{thm:rnn}
\end{restatable}
\vspace{-2em}
Corollary~\ref{thm:rnn} shows that, in the worst case, the policy gradient computed with an AWM instantiated with an RNN backbone can explode exponentially fast with respect to the problem horizon if the spectral radius of $W_x$ exceeds $\beta$. Just like Theorem~\ref{thm:hist-explode}, the above bound can be written as a lower bound when all gradients are positive scalars, therefore it is tight.
This result will allow us to compare the gradient provided by an RNN AWM with the one provided by a transformer AWM, but it is also connected with gradient computation with Markovian models: as seen in Proposition~\ref{pro:rnn_equivalence}, using a Markovian model to compute the policy gradient, either given or learned, implies an RNN-structure for the AWM. 
Theorem~\ref{thm:rnn} is a consequence of foundational theory of RNNs~\citep{BengioSF94,PascanuMB13}, thus formalizing the intuition that the recurrent nature of policy gradients through Markovian models make both exploiting differentiable simulators $f$ and learned models $\hat f_\vpsi$  difficult~\citep{metz2021gradients, HeessWSLET15}.

Let us now apply Theorem~\ref{thm:general-bound} to a layer of self-attention AWM~\citep{Vaswani2017}. Given $Q \in \mathbb{R}^{1\times d_z}$, $K \in \mathbb{R}^{n\times d_z}$, $V \in \mathbb{R}^{n \times d_o}$, we define $\text{Attention}(Q, K, V) := \text{softmax}(QK^T)V$. Using a set of weight matrices, a self-attention AWM predicts the next state as follows\footnote{Without loss of generality, the dependency on the initial state is not included in the self-attention layer.}:
\begin{flalign} \label{eq:att}
  \text{(}g_{\texttt{ATT}}\text{)} && \hat{s}_{t+1} =  \text{Attention}(a_{t}W_q^T, a_{1:t}W_k^T, a_{1:t}W_v^T) \enspace . &&
\end{flalign}
We can now characterize the behavior of the policy gradient computed with a self-attention AWM.
\begin{restatable}{corollary}{attBound}
Let $g_{\texttt{ATT}}$ be an attention-based Actions World Model instantiated with self-attention as in \eqref{eq:att}.
     The asymptotic behavior of the norm of the policy gradient $\nabla_{\vtheta} J^{g_\texttt{ATT}} (\vtheta; H)$ as a function of the horizon $H$ can be described as:
    \begin{align*}
        \left \|\nabla_{\vtheta} J^{g_\texttt{ATT}} (\vtheta; H) \right \| = O\left ( H^3\right).
    \end{align*}
\end{restatable}
This result shows that the norm of the policy gradient computed through a self-attention AWM has a worst-case polynomial dependency on the horizon instead of an exponential one shown in Theorem~\ref{thm:hist-explode} and Corollary~\ref{thm:rnn}. Since the output of a transformer model only depends on the sequence length through the attention mechanism, this characterization captures the gradient dynamics of the transformer model itself. It shows that, by eliminating all circuitous gradient paths, AWMs provide a framework for transformers to translate their favorable gradient properties towards better policy gradients, unlike state-conditioned models.

\begin{figure*}[t]
    \centering
    \subfloat[One-bounce environment overview.]
    {{\includegraphics[width=3.7cm]{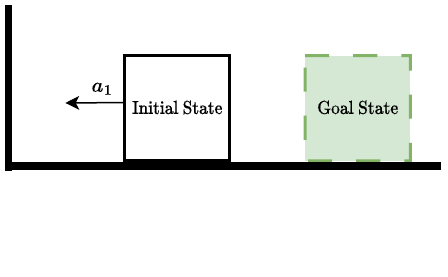}\label{fig:one-bounce-vis}}}
    \hfill
    \subfloat[Example trajectory with non-differentiable point at $t\approx 0.25$.]{\includegraphics[height=2.8cm]{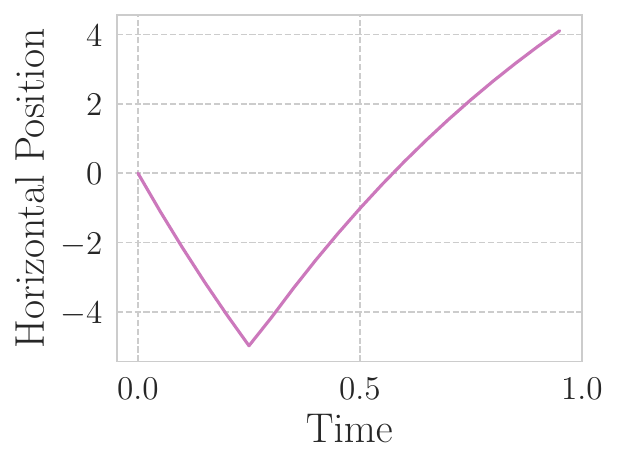}\label{fig:one-bounce-trajectory}}
    \hfill
    \subfloat[Final return with respect to the initial action for different models.]{\includegraphics[height=2.8cm]{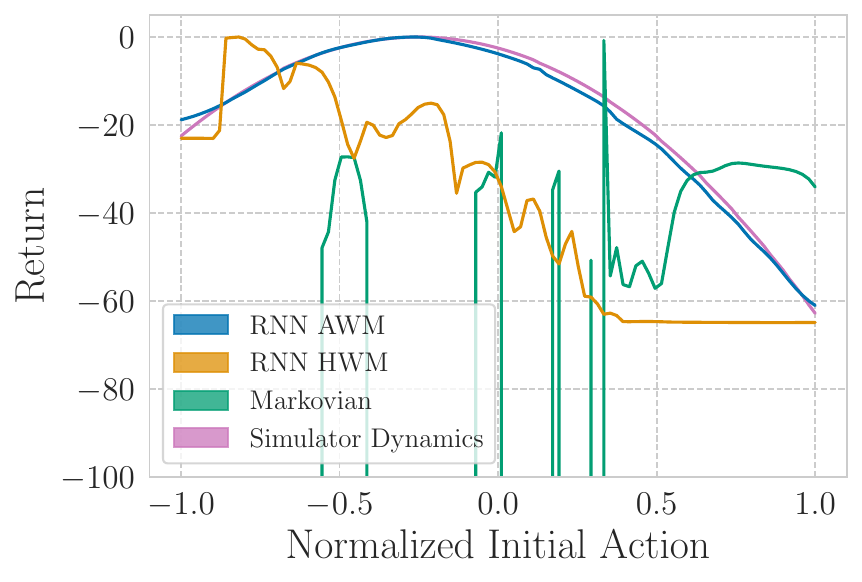}\label{fig:one-bounce-return}}
    \vspace{-0.5em}
    \caption{\footnotesize \textbf{AWMs ignore non-differentiable points in the state space.} (a) After the block is pushed with some initial action, it bounces off the wall, instantaneously reversing its velocity. (b) Visualization of the point of non-differentiability in the state space.  (c) Learning a Markovian model or a HWM causes catastrophic compounding errors, but an AWM can still accurately model the final reward when varying the initial action. Learned dynamics are trained offline on a dataset collected using random actions. }
    \label{fig:one-bounce}
    \vspace{-1.5em}
\end{figure*}

\section{Experiments} \label{sec:experiments}
Backpropagation through Actions World Models produces policy gradients with no circuitous gradient paths. Consequently, the longest gradient path from an action to any reward strictly depends on the AWM internals, while the path through state-based world models scales with the horizon, regardless of the model architecture. This section empirically explores how this phenomenon benefits the policy optimization process (see Appendix \ref{app:hparams} for hyperparameters of all experiments).

\subsection{Empirical Properties of Action World Models}

By removing the need to explicitly model the state dynamics auto-regressively, AWMs can internally build a more appropriate world model for policy optimization. Using environments inspired by real-world physical systems, we highlight two particular properties of AWMs: skipping points of complex state-space dynamics and smoothing out ill-behaved policy optimization landscapes.

\textbf{AWMs overcome non-differentiable dynamics}~~Unlike state-based world models, AWMs directly map actions to future states without explicitly predicting intermediate states. This mapping allows AWMs to skip points of potentially complex dynamics in the state space, which are not always necessary for policy optimization. For example, consider points of non-differentiability (e.g., contact points) in real-world systems (e.g., physical systems). These points make directly using the simulator dynamics for BPO impossible and can cause catastrophic compounding errors for state-conditioned world models. However, these points often exist even if the dependency of rewards on actions is smooth and well-behaved everywhere. In such cases, AWMs can be used to side-step the non-differentiability of the state dynamics. To illustrate this point, we examine a \emph{one-bounce} environment, shown in Figure \ref{fig:one-bounce}. In this task, an agent must push a block towards a wall with some initial action, such that the block bounces off the wall and ends up at some predetermined goal state after $H$ steps. A single terminal reward is given at $t=H$ measuring the distance of the block to the goal state. Even though the state trajectory is non-differentiable due to the wall bounce, the final state of the block is actually well-behaved with respect to initial actions (Figure \ref{fig:one-bounce-trajectory} and Figure \ref{fig:one-bounce-return}). Consequently, we show in Figure \ref{fig:one-bounce-return}
that AWMs can accurately predict the final reward while an unrolled Markovian or HWM cannot. Notably, the resulting AWM thus allows to easily use backpropagation-based policy optimization even in the presence of an underlying non-differentiable environment dynamics.

\textbf{AWMs overcome chaotic dynamics}~~
In many real-world systems (e.g., robots), the chaotic dynamics of an environment can create a remarkably complex mapping from actions to rewards over time or a highly non-smooth return landscape~\citep{rahn2023policy}.
In such systems, even when a differentiable simulator is available, it gives ill-behaved gradients and generates an optimization landscape that is hard to navigate~\citep{SuhSZT22}.
We now show that transformer AWMs are able to take full advantage of the inductive biases of the underlying neural networks, naturally generating policy optimization landscapes that are easy to navigate. This happens regardless of the complexity of the real dynamics while preserving the global problem structure and allowing good policies to be found more easily.
To demonstrate this, we use a prototypical chaotic system, a double-pendulum task depicted in Figure~\ref{fig:double-pend-vis}.
In this environment, the agent's initial action determines the initial angular position of the inner pendulum. Afterwards, the system is rolled out for $H$ steps. The objective is to get the pendulums to be in some predefined goal state after $H$ steps. A terminal reward is given measuring the distance between the final observed state and the desired goal state. We illustrate in Figure~\ref{fig:double-pend-grads} and Figure~\ref{fig:double-pend-returns} the chaotic nature of the task, showing respectively that the norm of the true gradient of the return with respect to the initial action grows exponentially with the episode length $H$, and that the true return landscape is extremely difficult to navigate for long horizons. 
Figure~\ref{fig:double-pend-grads} shows that transformers provide stable gradients when used as AWMs, while the gradient provided by HWMs explodes similarly to the true gradients. 
AWMs generate a smooth and accurate approximation of the return landscape (Figure \ref{fig:double-pend-returns}), which gradient ascent can easily navigate, leading to better policies when used for backpropagation-based policy optimization compared to all the other world models~(Figure~\ref{fig:double-pend-rl}).

\begin{figure*}[t]
    \centering
    \subfloat[Double-pendulum.]
    {{\includegraphics[height=2.3cm]{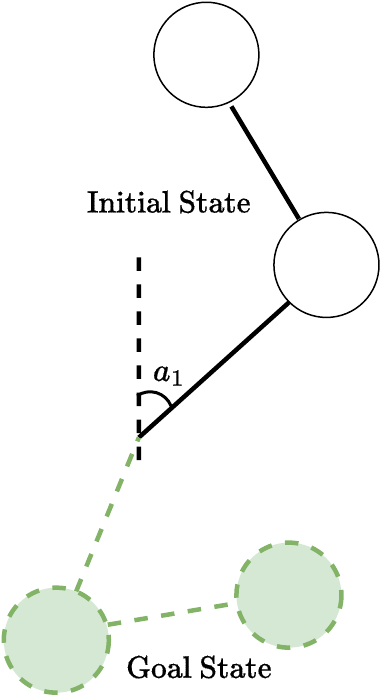}\label{fig:double-pend-vis}}}
    \hfill
    \subfloat[Gradient norms.]{\includegraphics[height=2.2cm]{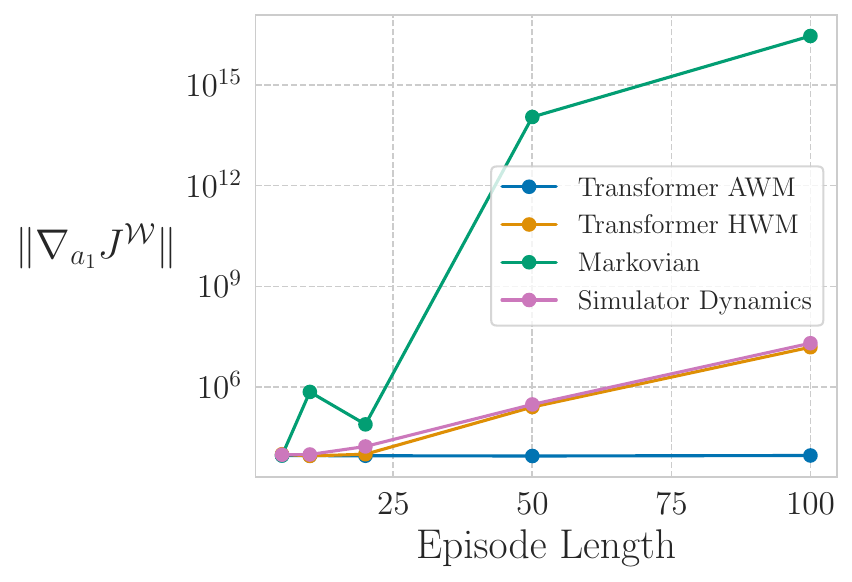}\label{fig:double-pend-grads}}
    \hfill
    \subfloat[Final returns with respect to the initial action for different models when $H=100$.
    ]{\includegraphics[height=2.2cm]{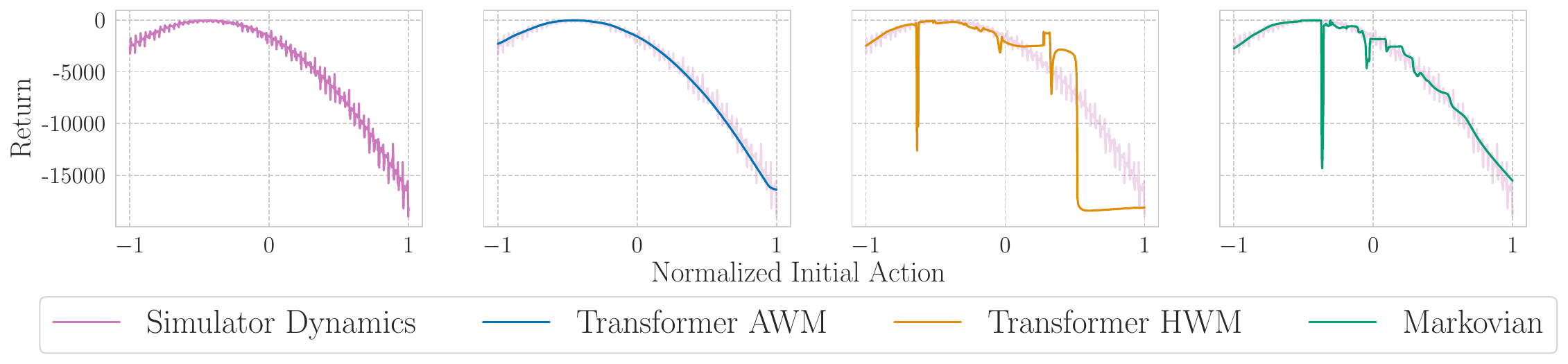}\label{fig:double-pend-returns}}
    \caption{\footnotesize \textbf{Transformer AWMs smooths out chaotic dynamics.} (a) A double-pendulum environment where an initial position must be chosen in order to achieve some pre-determined goal state after $H$ steps. Different transition models are learned on a data set of random trajectories. (b) The mean gradient norm of the final state with respect to the initial action for each model is computed over 50 different random actions for different horizons. (c) Final return according to different models with respect to different initial actions for $H=100$.}
    \label{fig:double-pend}
    \vspace{-1.5em}
\end{figure*}

\subsection{Benchmarking Credit Assignment with AWMs}
We now show that transformers employed as AWMs can give better policies in environments which require long-horizon planning. We focus on eight tasks from the Myriad testbed~\cite{howe_myriad}. Myriad provides continuous MDPs inspired by real-world problems \cite{lenhart} with configurable horizons. In Myriad, the agent requires better temporal credit assignment capabilities to solve problems with longer episode lengths, in contrast to other RL benchmarks (e.g., focused on simple robotics locomotion), which typically feature a short effective horizon~\cite{laidlaw2023bridging, ni2023transformers} regardless of the episode length.
Additional information about Myriad and its experiments are provided in Appendix~\ref{app:myriad_details}. All aggregate results use the IQM \cite{AgarwalSCCB21}.

\begin{figure*}[t]
\vspace{-2.em}
    \centering
    \subfloat[Final performances of BPO for different horizons.]
    {{\includegraphics[width=0.39\linewidth]{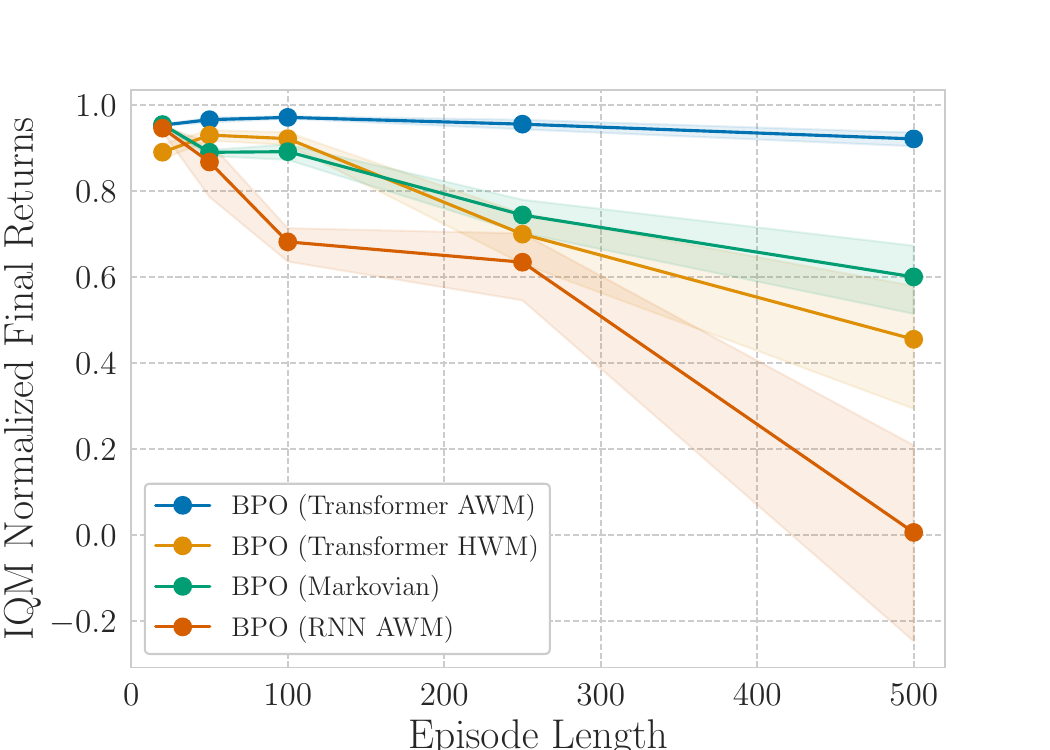}\label{fig:agg-myriad}}}
    \hfill
    \subfloat[Learning curves compared to other baselines.]{{\includegraphics[width=0.58\linewidth]{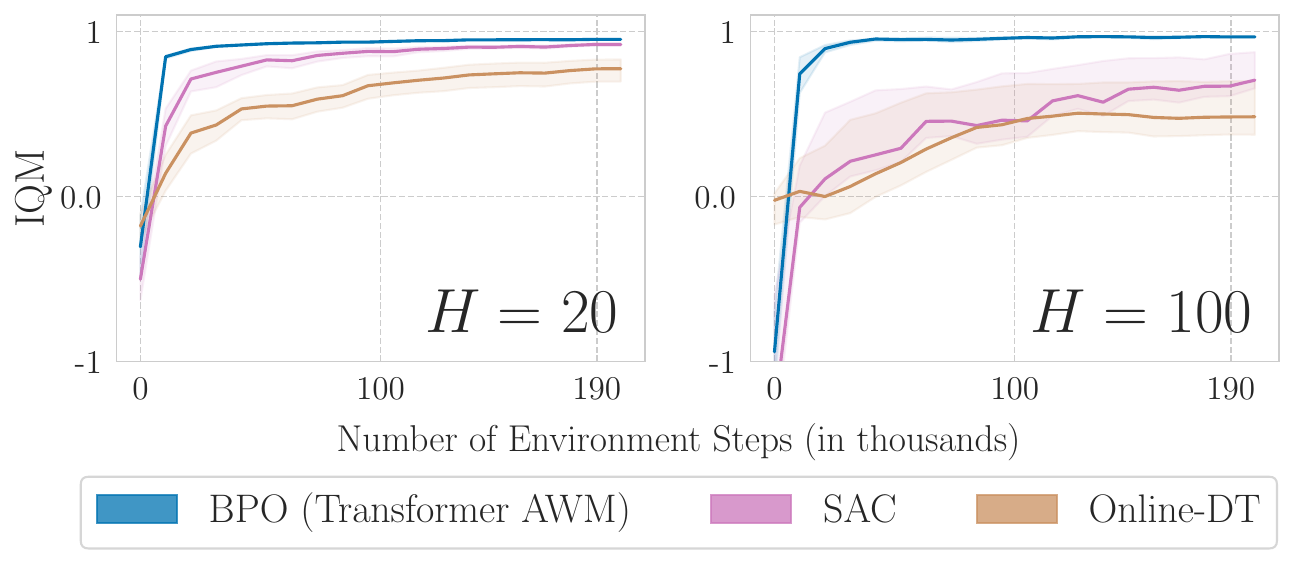}\label{fig:sample-efficiency}}}
     \caption{\footnotesize \textbf{Policy optimization with transformer AWMs gives better policies for long horizons.} (a) Final performance of BPO with different world models on Myriad (10 seeds $\pm$ 95\% C.I.). (b) Learning curves of BPO through a transformer AWM, a SAC agent, and an Online-DT agent on 20 and 100 length horizons (10 seeds $\pm$ 95\% C.I.).}
\end{figure*}
Figure~\ref{fig:agg-myriad} shows the final performance of various world models used in BPO with increasing episode lengths.  The transformer AWM is the only model to perform reliably well for increasingly long horizons, while other methods struggle when the horizon exceeds 100 steps. Indeed, when used as AWMs, transformer world models give better policy gradients than standard Markovian models and 
RNN world models. However, when transformer world models are conditioned on states, their long-term policy gradients are no better than the gradients induced by Markovian models as measured by performance in Figure~\ref{fig:agg-myriad}. These results corroborate our previous findings: transformer HWMs, RNN AWMs, and Markovian  models all generate long policy gradient paths, which makes long-term credit assignment difficult.

We also compare our method with other relevant baselines in Figure \ref{fig:sample-efficiency}, namely model-free Soft Actor-Critic (SAC)~\cite{haarnoja2018sac} and the recently popularized Online Decision Transformer (Online-DT)~\cite{zheng2022online,schmidhuber2019reinforcement,srivastava2019training}. In addition to BPO with transformer AWMs scaling better with the problem horizon, BPO is also significantly more sample efficient than the two baselines.

\vspace{-1em}
\section{Related Work}
\textbf{Backpropagation-based Policy Optimization}~~
The concept of learning a policy by differentiating through the environment dynamics has been a foundational aspect of early RL methods~\citep{werbos1974beyond,schmidhuber1990making}. Estimators for such a policy gradient have been referred to with different names, including pathwise derivatives~\citep{ClaveraFA20,hafner2022mastering,hafner2023mastering} and value gradients~\citep{fairbank2014value,HeessWSLET15,pmlr-v144-amos21a}.
AWMs connect this literature to the one on sequence modeling with neural networks, with a focus on gradient propagation~\citep{Hochreiter1991UntersuchungenZD,BengioSF94,hochreiter1998vanishing,PascanuMB13} and on architectures that allow for improved credit assignment~\citep{Hochreiter1997,Vaswani2017,KergKGGBL20}. Finally, our work can be seen as an instance of gradient-aware model-based RL~\citep{DOroModel2020,DOroMage20,Abachi2020PolicyAwareML}, in which we modify the architecture of a model of the dynamics to obtain better policy gradients~\citep{ma2021longterm}.


\textbf{Multi-step state prediction with actions}~~
Prior works used multi-step models related to AWMs, analyzing their role as partial models~\citep{rezende_causallycorrect}, in the context of tree search~\citep{schrittwieser_muzero}, or training them with a multi-step latent prediction loss~\citep{gregor2019, schwarzer_spr}, without however ever analyzing their gradient properties. Notably, models used in these prior works are rarely shown to work beyond a few dozen steps of unrolling or imagination, while AWMs demonstrably scale favorably for horizons well beyond established limitations. 

\textbf{Sequence models in RL}~~In partially observable MDPs, sequence models have been extensively used as history encoders to maximize returns with or without world models~\cite{hausknecht2015deep, ni2021recurrent, hafner2023mastering}. Other works have recently emerged treating MDPs as sequence modeling~\cite{ChenLRLGLASM21, zheng2022online, janner2021sequence}, in imitation learning or offline RL problem settings, based on return-conditioned models. 
Sequence models have also been used to reshape rewards in sparse reward settings~\cite{hung2019optimizing, arjona2019rudder, liu2019sequence}. In contrast to all of these, our framework uses an \textit{action-only} conditioned model to directly improve long-term policy gradients and naturally assigns credit by backpropagation as opposed to reward-shaping heuristics. Appendix~\ref{app:seq-rl} contains a more detailed discussion of how all these methods differ from ours.
\vspace{-1em}
\section{Conclusions and Limitations}
We presented a model-based policy optimization framework based on 
backpropagation through Actions World Models: models of the dynamics of the environment which only take an initial state and a sequence of actions as inputs.
We highlighted that, in contrast to the more common world models that compute the gradient by autoregressively unrolling their state predictions, AWMs do not suffer from the existence of circuitous gradient paths where the gradient can flow outside of the underlying network architecture.
We showed that their structure allows AWMs to fully exploit transformers' favorable properties, taking advantage, both in theory and in practice, of their ability to perform long-term temporal credit assignment by leveraging short gradient paths.

Theoretically, we demonstrated how transformer AWMs do not suffer from ill-behaved gradients coming from circuitous gradient paths which harm the quality of the gradients provided by more common Markovian models and HWMs.
Empirically, we showed that BPO using AWMs can be successful even when BPO using the underlying simulator fails, due to how AWMs leverage transformers' inductive biases to obtain an easier-to-navigate policy optimization landscape.
We demonstrated that policies trained with transformer AWMs outperform existing model-based and model-free approaches, including recently proposed algorithms based on transformers, in realistic domains requiring long-term planning.

While our theoretical analysis is developed within a deterministic setting, and our experiments are conducted in low-dimensional yet realistic domains, our framework establishes a new connection between gradient propagation techniques for sequence modeling and backpropagation-based policy gradients. Future work could focus on scaling our approach, making RL even more relevant to real-world applications where long-horizon capabilities are crucial. 

\subsubsection*{Acknowledgements and Disclosure of Funding}
The authors thank 
Nathan Rahn and Jürgen Schmidhuber for insightful discussions and useful suggestions on the paper, the Mila community for creating a stimulating research environment,
Digital Research Alliance of Canada
and Nvidia
for computational resources.
This work was partially supported by 
CIFAR,
Google, 
Facebook AI Chair, 
IVADO,
and Gruppo Ermenegildo Zegna.

{\small
\bibliography{iclr2024_conference}
\bibliographystyle{iclr2024_conference}
}

\clearpage

\addcontentsline{toc}{section}{Appendix} 
\part{Appendix} 
\parttoc 

\clearpage
\appendix



\section{Proofs}
\label{sec:proofs}

\histExplode*
\begin{proof}
    The policy gradient computed through equation \ref{eq:historyobj} can be expanded using the chain rule into the following expression:
    \begin{align*}
        \nabla_\vtheta J^h (\vtheta; H)&= \sum_{t=1}^H \frac{\partial r(\hat{s}_t)}{\partial \hat{s}_t} \sum_{k=1}^{t-1} \frac{dh(\hat{s}_{1:t-1}, a_{1:t-1})}{da_k} \frac{\partial \pi_\vtheta (\hat{s}_k)}{\partial \vtheta}\\
        \frac{dh(\hat{s}_{1:t-1}, a_{1:t-1})}{da_k}&= \frac{\partial h(\hat{s}_{1:t-1}, a_{1:t-1})}{\partial a_k} + \sum_{i=k+1}^{t-1} \frac{\partial h(\hat{s}_{1:t-1}, a_{1:t-1})}{\partial \hat{s}_i}\frac{dh(\hat{s}_{1:i-1}, a_{1:i-1})}{da_k}
    \end{align*}
    We begin by showing how the norm of the policy gradient depends on $\frac{dh(\hat{s}_{1:t}, a_{1:t})}{da_k}.$
    \begin{align*}
        \| \nabla_\vtheta J^h(\vtheta; H) \| &= \left \| \sum_{t=1}^H \frac{\partial r(\hat{s}_t)}{\partial \hat{s}_t} \sum_{k=1}^{t-1} \frac{dh(\hat{s}_{1:t-1}, a_{1:t-1})}{da_k} \frac{\partial \pi_\vtheta (\hat{s}_k)}{\partial \vtheta} \right\| \\
         \text{By Lipschitz assumption} \enspace &\leq \left \| \sum_{t=1}^H L_r \sum_{k=1}^{t-1} \frac{dh(\hat{s}_{1:t-1}, a_{1:t-1})}{da_k} L_\pi \right \| \\
         \text{Triangle inequality} \enspace &\leq  \sum_{t=1}^H L_r \sum_{k=1}^{t-1} \left \| \frac{dh(\hat{s}_{1:t-1}, a_{1:t-1})}{da_k} \right \| L_\pi \enspace .
    \end{align*}
    Now, we show how $\|\frac{dh(\hat{s}_{1:t-1}, a_{1:t-1})}{da_k} \|$ grows asymptotically with $t$ and $k$.
    \begin{align*}
        \left \|\frac{dh(\hat{s}_{1:t-1}, a_{1:t-1})}{da_k} \right \| &= \left \|\frac{\partial h(\hat{s}_{1:t-1}, a_{1:t-1})}{\partial a_k} + \sum_{i=k+1}^{t-1} \frac{\partial \hat{s}_t}{\partial \hat{s}_i}\frac{dh(\hat{s}_{1:i-1}, a_{1:i-1})}{da_k} \right \|\\
        \text{Triangle inequality} \enspace & \leq \left \|\frac{\partial h(\hat{s}_{1:t-1}, a_{1:t-1})}{\partial a_k}\right \| + \sum_{i=k+1}^{t-1} \left\| \frac{\partial h(\hat{s}_{1:t-1}, a_{1:t})}{\partial \hat{s}_i}\frac{dh(\hat{s}_{1:i-1}, a_{1:i-1})}{da_k} \right\|\\
        \text{Cauchy-Schwarz inequality} \enspace & \leq \left\|\frac{\partial h(\hat{s}_{1:t-1}, a_{1:t-1})}{\partial a_k} \right\| + \sum_{i=k+1}^{t-1} \left\| \frac{\partial h(\hat{s}_{1:t-1}, a_{1:t-1})}{\partial \hat{s}_i}\right\| \left\| \frac{dh(\hat{s}_{1:i-1}, a_{1:i-1})}{da_k} \right\|\\
        \text{By assumption} &\leq L_a + \sum_{i=k+1}^{t-1} L_s \left\| \frac{dh(\hat{s}_{1:i-1}, a_{1:i-1})}{da_k} \right\| \enspace. 
    \end{align*}
    In the recursive expansion of the previous line, notice the linear dependency w.r.t. $L_a$, and exponential dependency w.r.t. $L_s$ due to the compounding products. Therefore, $\|\frac{dh(\hat{s}_{1:t}, a_{1:t})}{da_k} \| = O(L_a + L_s^{t-k})\enspace .$
    We can then put everything together to derive the asymptotic behavior.
    \begin{align*}
        \| \nabla_\vtheta J^h(\vtheta; H) \| &\leq  \sum_{t=1}^H L_r \sum_{k=1}^{t-1} \left\| \frac{dh(\hat{s}_{1:t-1}, a_{1:t-1})}{da_k} \right\| L_\pi\\
        &= O\left(HL_r +H^2 L_\pi + H^2 \left\| \frac{dh(\hat{s}_{1:H-1}, a_{1:H-1})}{da_1} \right\|\right)\\
        &= O(HL_r +H^2 L_\pi + H^2 L_a + H^2 L_s^H)
    \end{align*}
\end{proof}

\rnnEquivalence*
\begin{proof}
    We begin by providing a formal definition of the recurrent network with its recurrent cell being the dynamics $f$ of the MDP as:
    \begin{align*}
        g_{f\text{-}\texttt{RNN}}(s_1, a_{1:t}) &:= f(s_t, a_t)\\
        s_t&:=f(s_{t-1}, a_{t-1})=f(g_{f\text{-}\texttt{RNN}}(s_1, a_{1:t-1}), a_t)
    \end{align*}
    Begin by developing the LHS of the equation.
    \begin{align*}
        \nabla_\vtheta^{g_{f \text{-}\texttt{RNN}}} J(\vtheta; H) &=
         \sum_{t=1}^H \frac{\partial r(s_t)}{\partial s_t} \sum_{k=1}^{t-1} \frac{\partial s_t}{\partial a_k} \frac{\partial \pi_\vtheta(s_k)}{\partial \vtheta}
    \end{align*}
    where $s_t=f(s_{t-1}, a_{t-1})$ by construction (RNN assumption). Therefore, apply the chain rule $\frac{\partial s_t}{\partial a_k}$:
    \begin{align*}
        \nabla_\vtheta^{g_{f\text{-}\texttt{RNN}}} J(\vtheta; H) &= \sum_{t=1}^H \frac{\partial r(s_t)}{\partial s_t} \sum_{k=1}^{t-1} \frac{\partial s_{k+1}}{\partial a_k} \bigg(\prod_{i={k+1}}^{t-1} \frac{\partial s_{i+1}}{\partial s_i}\bigg)\frac{\partial \pi_\vtheta(s_k)}{\partial \vtheta}
    \end{align*}
    Now, for the RHS of the equation. The policy gradient through equations \ref{eq:mdpobj} can be written as:
    \begin{align*}
        \nabla_\vtheta J^f (\vtheta;H) &= \sum_{t=1}^H \frac{\partial r(s_t)}{\partial s_t} \sum_{k=1}^{t-1} \frac{\partial s_{k+1}}{\partial a_k}\frac{\partial \pi_\vtheta (s_k)}{\partial \vtheta} \left(\prod_{i={k+1}}^{t-1} \frac{\partial s_{i+1}}{ \partial s_i} \right).\\
        \nabla_\vtheta^{g_{f \text{-}\texttt{RNN}}} J(\vtheta; H) &=  \nabla_\vtheta J^f (\vtheta;H).
    \end{align*}
\end{proof}

\generalBound*
\begin{proof}
With a slight abuse of notation, we use $g_t$ to replace $g(s_1, a_{1:t-1})$.
    \begin{equation*}
        \nabla_\vtheta J^g(\vtheta; H) = 
        \left \| \sum_{t=1}^H \frac{\partial r(\hat{s}_t)}{\partial \hat{s}_t} \sum_{k=1}^{t-1} \frac{\partial g_t}{\partial a_k} \frac{\partial \pi(\hat{s}_k)}{\partial \vtheta} \right \| 
        \leq L_r L_\pi \left\| \sum_{t=1}^H \sum_{k=1}^{t-1} \frac{\partial g_t}{\partial a_k} \right\| \leq L_r L_\pi  \sum_{t=1}^H \sum_{k=1}^{t-1} \left\| \frac{\partial g_t}{\partial a_k} \right\|.
    \end{equation*}
\end{proof}

\rnnBound*
\begin{proof}
    We begin by defining the transition function of the recurrent network as $x_t = \sigma(W_x x_{t-1}) + W_a a_{t} + b$, where $\sigma$ is some activation function with gradient norm bounded by $\left \| diag(\sigma'(x))\right \| \leq \frac{1}{\beta}$. Further, consider a linear output cell $s_{t+1} = W_o x_{t}$. Now we begin by showing that $\left \| \frac{\partial x_t}{\partial a_k} \right \| \leq ||W_a^T || (||W_x^T|| \frac{1}{\beta})^{t-k}$: 

    \begin{align*}
        \frac{\partial x_t}{\partial a_k} &= \frac{\partial x_k}{\partial a_k} \prod_{i=k}^{t-1} \frac{\partial x_{i+1}}{\partial x_i} \\
        \left \| \frac{\partial x_t}{\partial a_k}\right \| & =  \left \| W_a^T \prod_{i=k}^{t-1}  \frac{\partial x_{i+1}}{\partial x_i}\right \| \\
        \text{Cauchy-Schwarz Inequality} \enspace &\leq || W_a^T|| \prod_{i=k}^{t-1} \left \| \frac{\partial x_{i+1}}{\partial x_i}\right \| = || W_a^T|| \prod_{i=k}^{t-1} \left \| W_x^T diag(\sigma'(x_i))\right \|\\
        \text{Cauchy-Schwarz Inequality} \enspace& \leq || W_a^T|| \prod_{i=k}^{t-1} \left \| W_x^T \right \| \left \|diag(\sigma'(x_i))\right \|\\
        \text{By assumption} \enspace & \leq || W_a^T|| \prod_{i=k}^{t-1} \left \| W_x^T \right \| \frac{1}{\beta} = ||W_a^T || \left(||W_x^T|| \frac{1}{\beta}\right)^{t-k} \enspace .
    \end{align*}

    We develop the expression $\frac{\partial \hat{s}_t}{\partial a_k}=\frac{\partial \hat{s}_t}{\partial x_{t-1}}\frac{\partial x_{t-1}}{\partial a_k}$, applying the norm and plugging in the above result we obtain:
    \begin{align*}
        \left \| \frac{\partial \hat{s}_t}{\partial a_k}\right \| \leq \left \|  W_o\right \| \left \|W_a^T \right \| \left(||W_x^T|| \frac{1}{\beta}\right)^{t-k-1}
    \end{align*}
    We then use the bound found in Theorem 1:
    \begin{align*}
        \left \| \nabla_\vtheta J^g(\vtheta; H)\right\| &\leq L_r L_\pi \sum_{t=1}^H \sum_{k=1}^{t-1} \left \| \frac{\partial g_t}{\partial a_k} \right \| =  L_r L_\pi \sum_{t=1}^H \sum_{k=1}^{t-1} \left \| \frac{\partial \hat{s}_t}{\partial a_k} \right \|\\
        &\leq   L_r L_\pi \sum_{t=1}^H \sum_{k=1}^{t-1} \left \|  W_o\right \| \left \|W_a^T \right \| \left(||W_x^T|| \frac{1}{\beta}\right)^{t-k-1}= L_r L_\pi \sum_{t=1}^H \sum_{k=1}^{t-1} \left \|  W_o\right \| \left \|W_a^T \right \| \eta^{t-k-1}\enspace ,
    \end{align*}
    where $\eta=||W_x^T|| \frac{1}{\beta}$. 
    \begin{align*}
        \left \| \nabla_\vtheta^{g} J(\vtheta; H)\right\| &\leq L_r L_\pi \sum_{t=1}^H \sum_{k=1}^{t-1} \left \|  W_o\right \| \left \|W_a^T \right \| \eta^{t-k-1} \\
        &= O(H^2||W_o|| \left \| W_a^T\right\| + H^2\eta^H)\\
        &= O(\eta^H) \enspace .
    \end{align*}
\end{proof}

\attBound*
\begin{proof}
    We begin with the definition of the self-attention Actions World Model. Let $Q \in \mathbb{R}^{1\times d_z}$, $K \in \mathbb{R}^{n\times d_z}$, $V \in \mathbb{R}^{n \times d_o}$, attention can be defined as :
$$\text{Attention}(Q, K, V) := \text{softmax}(QK^T)V$$

A self-attention AWM can then be defined in the following way with weight matrices $W_q \in \mathbb{R}^{d_z \times d_a}$, $W_k \in \mathbb{R}^{d_z \times d_a}$, $W_v \in \mathbb{R}^{d_s \times d_a}$.

$$\texttt{ATT}_t(a_1, ...a_{t-1}):=\text{Attention}(a_{t-1}W_q^T, a_{1:t-1}W_k^T, a_{1:t-1}W_v^T)=
\sum_{i=1}^{t-1}c_i (a_i W_v^T) \enspace =\hat{s}_t,$$

where $c_i=\text{softmax}_i(a_{t-1}W_q^TW_k a_{1:t-1}^T)$. The subscript on the softmax operator indicates the $i$th index. Now, we begin by showing the expression for $\frac{\partial g_t}{\partial a_k}$:
    \begin{align*}
        \frac{\partial g_t}{\partial a_k} &= \sum_{i=1}^{t-1} \frac{\partial}{\partial a_k}c_i (a_i W_v^T)\\
        &= \sum_{i=1}^{t-1} \frac{\partial c_i}{\partial a_k}(a_iW_v^T) + c_i\frac{\partial (a_iW_v^T)}{\partial a_k}\\
        \text{By Softmax derivative} \enspace &=  c_kW_v^T  + \sum_{i=1}^{t-1}\bigg( c_i(1\{i=k\} - c_k)(a_iW_v^T)\bigg) \enspace .
    \end{align*}

Then, we take the norm:
\begin{align*}
    \left \| \frac{\partial g_t}{\partial a_k}\right \| &\leq \| c_kW_v^T\| + \sum_{i=1}^{t-1} \left \| c_i(1\{i=k\} - c_k)(a_iW_v^T)\right \| \enspace \text{Triangle Inequality}\\ 
    \text{Cauchy-Schwarz Inequality} \enspace &\leq \| c_kW_v^T\| + \sum_{i=1}^{t-1} \left \| c_i(1\{i=k\} - c_k) \right \| \left \|(a_iW_v^T)\right \|\\
    \text{By definition of Softmax} \enspace &\leq \| W_v^T\| + \sum_{i=1}^{t-1} \left \|(a_iW_v^T)\right \|  \\
    \text{Assuming bounded actions } \enspace & \leq  \| W_v^T\| + \alpha \sum_{i=1}^{t-1} \left \|W_v^T\right \| \enspace , \text{where $a_i \leq \alpha \enspace  \forall \enspace i$} \enspace .
\end{align*}
Finally, we use the bound derived in Theorem 1 to finalize the proof:
\begin{align*}
    \left \| \nabla_\vtheta J^g(\vtheta; H) \right \| &\leq L_r L_\pi \sum_{t=1}^H \sum_{k=1}^{t-1} \bigg(\| W_v^T \| + \alpha \sum_{i=1}^{t-1} \| W_v^T \|\bigg)\\
    &=O(H^3\alpha \| W_v^T \|)\\
    &=O(H^3) \enspace .
\end{align*}
\end{proof}

\subsection{Tightness of Theorem~\ref{thm:hist-explode}}
\label{tightness-hist}
The result of Theorem~\ref{thm:hist-explode} argues that policy gradients computed through equation~\ref{eq:historyobj} may explode as the horizon grows. To make this claim, we comment on the tightness of the upper bound provided in Theorem~\ref{thm:hist-explode}. Consider the one-dimensional case, where all the gradients are positive scalars, and $\| \frac{\partial h(\hat{s}_{1:t}, a_{1:t})}{\partial a_k} \| = L_a$ and $\| \frac{\partial h(\hat{s}_{1:t}, a_{1:t})}{\partial \hat{s}_i}\| = L_s$ for all $\hat{s}_{1:t}, a_{1:t},k, i$. Then, we show that 
\begin{align*}
    \| \nabla_\vtheta J^h(\vtheta; H)\| &= \Omega(L_s^H)
\end{align*}
\begin{proof}
We begin the proof of this lower bound in a similar way, by considering only one of the terms of the summation. For clarity, Let $h_i=h(\hat{s}_{1:i-1}, a_{1:i-1})$.
\begin{align*}
      \|\nabla_\vtheta J^h (\vtheta; H) \| &= \left \| \sum_{t=1}^H \frac{\partial r(\hat{s}_t)}{\partial \hat{s}_t} \sum_{k=1}^{t-1} \frac{dh_t}{da_k} \frac{\partial \pi_\vtheta (\hat{s}_k)}{\partial \vtheta} \right \|\\
      & \geq \left \| \frac{\partial r(\hat{s}_H)}{\partial \hat{s}_H} \frac{dh_H}{da_1} \frac{\partial \pi_\vtheta (\hat{s}_1)}{\partial \vtheta} \right\|\\
      &= \left \| \frac{\partial r(\hat{s}_H)}{\partial \hat{s}_H} \right\| \left\| \frac{dh_H}{da_1}\right\| \left\| \frac{\partial \pi_\vtheta (\hat{s}_1)}{\partial \vtheta} \right\|\\
      &\geq \left\| \frac{\partial r(\hat{s}_H)}{\partial \hat{s}_H} \right\| \left\| \frac{\partial \pi_\vtheta (\hat{s}_1)}{\partial \vtheta} \right\| 
      \left\| \bigg( \frac{\partial h_H}{\partial a_1} + \sum_{i=2}^{H-1} \frac{dh_H}{\partial s_i}\frac{dh_i}{da_1} \bigg) \right\|  \enspace .\\
\end{align*}
Consider now only the longest chain of products in the summation term of the last line. This corresponds to the longest path of the policy gradient through the history world model. 
\begin{align*}
    \|\nabla_\vtheta J^h (\vtheta; H) \| & \geq \left \| \frac{\partial r(\hat{s}_H)}{\partial \hat{s}_H} \right\| \left\| \frac{\partial \pi_\vtheta (\hat{s}_1)}{\partial \vtheta} \right\| 
      \left\| \bigg( \frac{\partial h_H}{\partial a_1} + \frac{dh_2}{\partial a_1}\prod_{i=2}^{H-1} \frac{dh_i}{\partial s_i} \bigg) \right\| \\
      &= \Omega(L_a + L_s^H) \enspace .
\end{align*}
\end{proof}
\subsection{Tightness of Corollary~\ref{thm:rnn}}
\label{tightness}
The result of Corollary~\ref{thm:rnn} argues that policy gradients through Markovian models may explode as the horizon grows. In order to make this claim, we comment on the tightness of the bound. Just like for Theorem \ref{thm:hist-explode}, consider the one-dimensional case, where all the gradients are positive scalars and assume further the gradient norm of the activation function $\sigma$ is lower bounded by $\gamma> 0$ (which is the case for popular activation functions such as sigmoid). We will show that, given these conditions, the bound can be written as a lower bound:
\begin{align*}
    \|\nabla_\vtheta J^{g_\texttt{RNN}}(\vtheta;H) \| = \Omega(\eta^H) \enspace .
\end{align*}
\begin{proof}
We begin the proof by lower bounding the expression for the recurrent cell $\frac{\partial x_t}{\partial a_k}$:
\begin{align*}
    \frac{\partial x_t}{\partial a_k} &\geq ||W_a^T || (||W_x^T|| \gamma)^{t-k}\\
    \left \| \frac{\partial \hat{s}_t}{\partial a_k} \right \| &\geq \| W_o \|||W_a^T || (||W_x^T|| \gamma)^{t-k} \enspace .
\end{align*}
The bound found in Theorem \ref{thm:general-bound} can also be rewritten as an equality given these additional assumptions, since the Cauchy-Schwarz and triangle inequalities become equalities in the one-dimensional positive case: $\| \nabla_\vtheta J^g(\vtheta; H) \| =L_r L_\pi  \sum_{t=1}^H \sum_{k=1}^{t-1} \left\| \frac{\partial g_t}{\partial a_k} \right\| $. Finally, putting these new results together, we find that 
\begin{align*}
    \|\nabla_\vtheta J^{g_{\texttt{RNN}}}(\vtheta; H) \| &\geq L_r L_\pi \sum_{t=1}\sum_{k=1}^{t-1} \|W_o\| \|W_a^T\| (\|W_x^T\|\gamma)^{t-k}\\
    &\geq L_r L_\pi \|W_o \| \|W_a^T\| (\|W_x^T\| \gamma)^{H-1}\\
    &= \Omega(\eta^H) \enspace ,
\end{align*}
where $\eta = \|W_x^T\|\gamma$.
\end{proof}

\section{Training an AWM with Teacher Forcing}
The policy gradient framework provided by Actions World Models can also be used to interpret training design choices used in RL through the lens of sequence modeling. 
As a concrete example, recall that the transition function of a Markovian model can be trained with the following loss:
\begin{align}
\label{eq:one-step-loss}
    \ell^{\hat{f}}(\tau; \vpsi) = \sum_{t=1}^{H-1} \|s_{t+1}^\tau - \hat{f}_\vpsi(s_t^\tau, a_t^\tau)\|^2 \enspace .
\end{align}
We present the following remark.
\begin{remark}
    Training an AWM instantiated with an RNN to predict states using teacher forcing is equivalent to training its recurrent cell with the loss function from Equation~\ref{eq:one-step-loss}.
    \label{rem:teacher}
\end{remark}
Teacher forcing~\citep{williams1989} is an algorithm used to train recurrent neural networks. Instead of autoregressively unrolling them during training, it enforces the outputs (i.e., in this case, the environment/RNN states) to be the ones coming from a ground-truth sequence.  
In other words, when training an AWM as an RNN with teacher forcing, we are essentially training its recurrent cell as a one-step model, as traditionally done in model-based RL applied to a Markovian setting. 

\section{Reinforcement Learning Hyperparameters}
\label{app:hparams}

\subsection{Model-based Hyperparameters}
All model-based backpropagation-based policy optimization methods, including the Markovian agent, AWMs and HWMs were sweeped on three entropy constants for exploration: $[0.1, 0.01, 0.001]$, and the best performing results are reported. Indeed, in addition to the objective function $J$, the policies are optimized with an additional entropy maximization factor as done in \citet{pmlr-v144-amos21a} for exploration. All actions, and states (for History World Models) are embedded with a linear layer with an output size of 72. The RNN is initialized with two hidden layers, with a hidden layer size of (64, 64). The self-attention transformers use a similar architecture to the GPT-2 model \citep{radford2019language} implemented by the Hugging Face Transformer library \citep{wolf2020huggingfaces}. Our transformers stack 2 layers and 3 heads of self-attention modules, with hidden layers of size 64. Timesteps are added as an input to every input of all world models, since we are in a finite-horizon setting. The Markovian world models predicts transitions as a difference function: $\hat{s}_t = \hat{f}(s_{t-1}, a_{t-1}) + s_{t-1}$, using two hidden layers of size (64, 64), and ReLU activation functions. Gradient norms are clipped at a value of 100 for all policy gradients. All other hyper-parameters relating to the policy learning algorithm are shown in Table \ref{tab:modle-based-hparams}.
\begin{table}[!hbt]
    \centering
    \begin{tabular}{c|c}
    \hline
       \textbf{Hyperparameter} &  \textbf{Value} \\
       \hline
       Number of Environment steps & 200000\\
        Dynamics replay ratio & 2\\
        Policy replay ratio & 16\\
        Dynamics batch size & 64\\
        Policy batch size  & 16\\
        Dynamics learning rate & 0.001\\
        Policy learning rate & 0.0001\\
        Replay buffer size & 1e6\\
        Warmup steps & 1500\\
        \hline
    \end{tabular}
    \caption{Hyper-parameters for all BPO algorithms that do not pertain to the world model hyper-parameters.}
    \label{tab:modle-based-hparams}
\end{table}

\subsection{Model-free Hyperparameters}
We use a model-free soft-actor critic \cite{haarnoja2018sac} as a benchmark for the myriad experiments. The critic of this agent is modeled as a 2-layer MLP with 256 hidden units each. Entropy regularization is done using a constant entropy constant. Results in Figure \ref{fig:sample-efficiency} show the best performing results after doing a grid search over the following hyper-parameters: learning rate $=[0.001, 0.0001, 0.00001]$ and entropy constant $=[1., 0.01, 0.001]$. The remaining hyperparameters are shown in Table \ref{tab:model-free-hparams}.
\begin{table}[!hbt]
    \centering
    \begin{tabular}{c|c}
    \hline
        \textbf{Hyperparameter} & \textbf{Value}\\
        \hline
        Number of environment steps & 200000\\
        Critic replay ratio & 2 \\
         Policy replay ratio& 16\\
         Batch size & 128\\
         Discount factor & 0.995\\
         Polyak averaging factor &  0.995\\
         Replay buffer size & 1e6\\
         Warmup steps & 1500\\
         \hline
    \end{tabular}
    \caption{Hyper-parameters for all model-free results on the Myriad environments.}
    \label{tab:model-free-hparams}
\end{table}

\section{Offline Experiments}
The experiments showed in Figure \ref{fig:one-bounce} and Figure \ref{fig:double-pend} are in the offline setting. In both settings, 100000 transitions are collected for all experiments using a uniform random policy. The world models are then trained on data sampled from these transitions for 1000 steps. The batch size and learning rate are set to the same values as shown in Table \ref{tab:modle-based-hparams}.

\section{Online Double-pendulum Experiments}
The goal position of the environment is calculated using some fixed initial angular position (initial action) for all episode lengths. In our experiments, the optimal action is -0.4, which corresponds to an initial angle of $-0.4 \times 180^\circ.$ The results in the double-pendulum experiments use the same hype-parameters as the Myriad experiments, with two exceptions: the number of environment steps used for optimization is 100000 instead of 200000, and the leaerning rate is sweeped for values in [0.01, 0.001, 0.0001, 0.00001, 0.000001] due to the unstable gradients. The complete learning curves are shown in Figure \ref{fig:curves-double-pendulum}.
\begin{figure}[!hbt]
    \centering
    \includegraphics[width=0.9\textwidth]{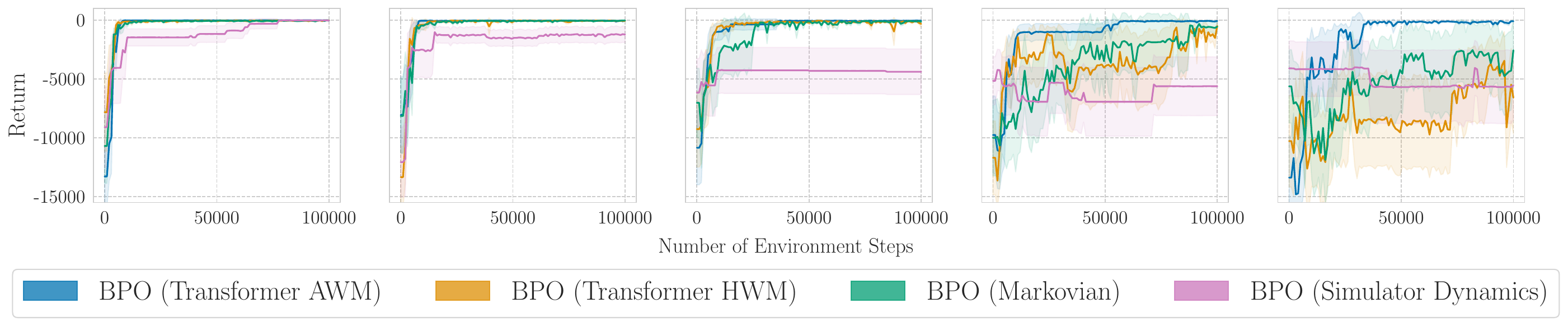}
    \caption{Learning curves of different world models used for BPO on the double-pendulum environment for different horizons. The horizons presented are $H=[5, 10, 20, 50, 100]$ in order (10 seeds $\pm$ std).}
    \label{fig:curves-double-pendulum}
\end{figure}
\section{Myriad Environments}
\label{app:myriad_details}
We give a short overview of each of the Myriad environments used in this work. For more details, refer to \cite{howe_myriad} and \cite{lenhart}.

The underlying dynamics of each of these environments are described by a set of ordinary differential equations, which are then discretized using Euler's method for discrete-time optimal control. In practice, to change a problem's horizon, we fix the duration of the experiment to ensure it is still interesting and simply discretize using a smaller or larger step size.
We normalize the returns, where $0$ is the expected performance of a random policy and $1$ is the performance of an optimal policy provided by \cite{howe_myriad}\footnote{Optimal policies are computed using trajectory optimization on the underlying differential equations.}. Our experiments are conducted on eight environments: \textit{cancer treatment, bioreactor, mould fungicide, bacteria, harvest, invasive plant, HIV treatment,} and \textit{timber harvest}. Each environment's state and action space varies between one and five dimensions.

\textbf{Cancer treatment} follows the normalized density of a cancerous tumour undergoing chemotherapy. The actions at every time step correspond to the strength of the chemotherapy drug at a given time. The goal is to minimize the size of the tumour over a set fixed duration while also minimizing the amount of drugs administered to the patient.

\textbf{Bioreactor} seeks to minimize the total amount of a chemical contaminant that naturally degrades in the presence of bacteria. The actions here allow the agent to feed the bacteria, increasing its population and increasing the rate of the contaminants' degradation. However, a cost is associated with feeding the bacteria.

\textbf{Mould fungicide} models the concentration of a mould population. The goal is to minimize its population by applying a fungicide, which has an associated cost to apply.

\textbf{Bacteria} looks to maximize a bacteria population through the application of a chemical nutrient that stimulates growth. On top of the associated cost of applying the chemical nutrient, the chemical also produces a byproduct that might, in turn, hinder bacterial growth.

\textbf{Harvest} models the growing population of some vegetables, and the goal is to maximize the harvested yield of this population. While harvesting directly contributes to the reward, it consequently slows down the population's exponential growth.

\textbf{Invasive plant} seeks to minimize the presence of an invasive plant species through interventions that remove a proportion of the invasive population. These actions have an associated cost.

\textbf{HIV treatment} follows the evolution of uninfected and infected cells in the presence of a virus. The actions correspond to a drug administered that affects the virus' rate of infection. The use of the drug must also be minimized.

\textbf{Timber harvest} is similar to the harvest environment, except the harvested population is infinite. Instead, harvested timber can be converted into capital, which can then be re-invested in the harvesting operation, stimulating company growth. The goal is to maximize revenue. 

To the best of our knowledge, this paper presents the first results using reinforcement learning for the environments in Myriad. The final performances of all BPO methods and the SAC agent for each of the eight environments are shown in Figures \ref{fig:individual_myriad_envs}. Note that not only does our method (BPO with transformer AWM) have the best aggregate scores for long horizons, but it also achieves the best performance for the longest horizon of 500 steps for each environment.
\begin{figure}[!hbt]
    \centering
    \includegraphics[width=0.95\textwidth]{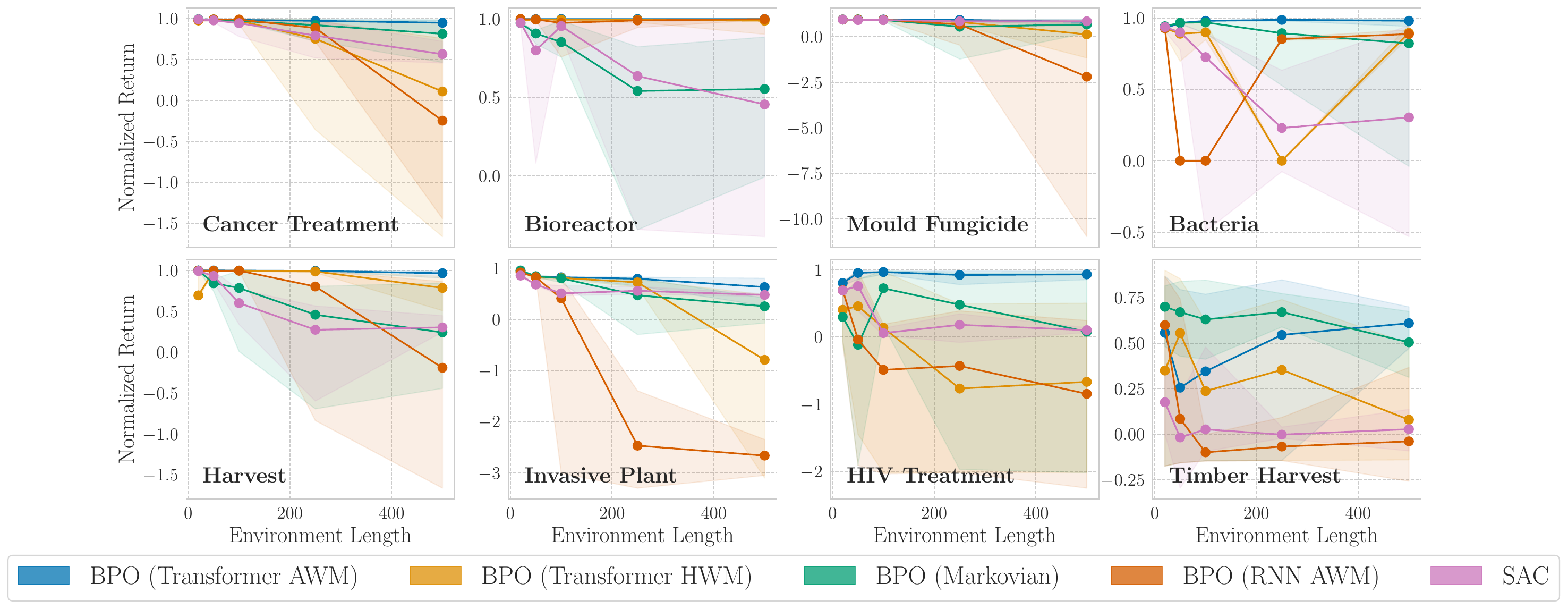}
    \caption{\textbf{Policy gradients through transformer AWMs give better policies for long horizons for every environment. } Learning curves of BPO with different world models on eight of the Myriad environments (10 seeds $\pm$ min/max)}.
    \label{fig:individual_myriad_envs}
\end{figure}

\subsection{Ablation on Stop Gradient}
\label{app:stopgrad-ablation}
Most of the analysis in this work relies on the assumption that the gradient is stopped through the policy input, as seen in equation \ref{eq:mdpobj}, equation \ref{eq:historyobj} and equation \ref{eq:awmrnn}. Just like it was noted in prior work~\cite{hafner2022mastering, hafner2023mastering, ghugare2023simplifying}, we also perform an ablation on whether the stop-gradient operator on the policy input has any major impact on the final performance on the Myriad benchmark. The results in Figure \ref{fig:fullgrad} confirm previous remarks that detaching the policy input from the policy gradient has little to no effect on final performance.

\begin{figure}[!hbt]
    \centering
    \includegraphics[width=0.5\textwidth]{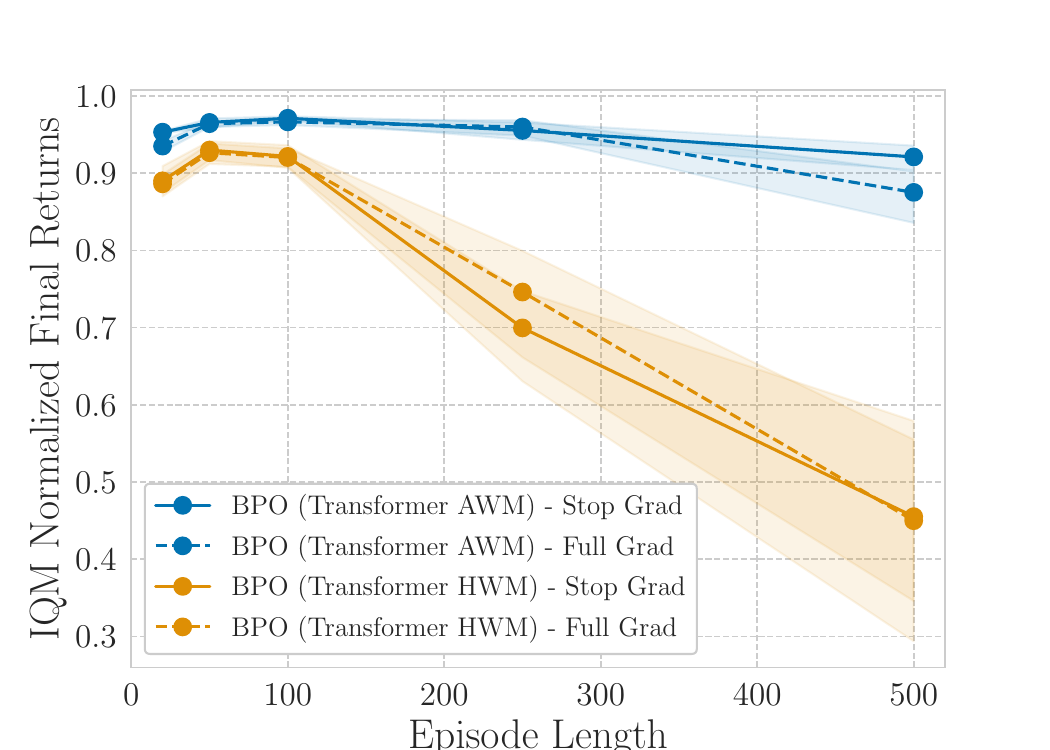}
    \caption{\textbf{Stop gradients on the policy input do not affect final performances on the Myriad benchmark.} Final performances of BPO with a transformer AWM and a transformer HWM, with and without the full policy gradient (10 seeds $\pm$ 95\% C.I.). \emph{Stop Grad} denotes the policy gradient used throughout the main text of this paper, while \emph{Full Grad} denotes the full unbiased policy gradient without any stop gradient operators.}
    \label{fig:fullgrad}
\end{figure}

\subsection{Ablation on LSTMs}
Our analysis suggests that policy gradients through an AWM inherit certain gradient properties of the underlying world model architecture. Long-short-term memory networks (LSTMs)~\cite{hochreiter1998vanishing} were also suggested as an alternative to vanilla RNNs due to their favorable long-term gradients. Indeed, we show in an ablation study that LSTM AWMs also perform quite well on the Myriad benchmark in Figure \ref{fig:lstm-myriad}. Thus, the Actions World Model is agnostic to the specific neural network architecture, allowing any modern sequence models to produce favorable policy gradients, which could be an interesting direction to explore other more advanced sequence models such as state space models~\cite{gu2022efficiently}. We hypothesize that LSTMs are quite effective on environments from Myriad due to the recurrent inductive bias that is compatible with all MDPs. We suspect transformer AWMs show greater promise for large-scale experiments, especially in POMDPs.
\begin{figure}[!hbt]
    \centering
    \includegraphics[width=0.9\textwidth]{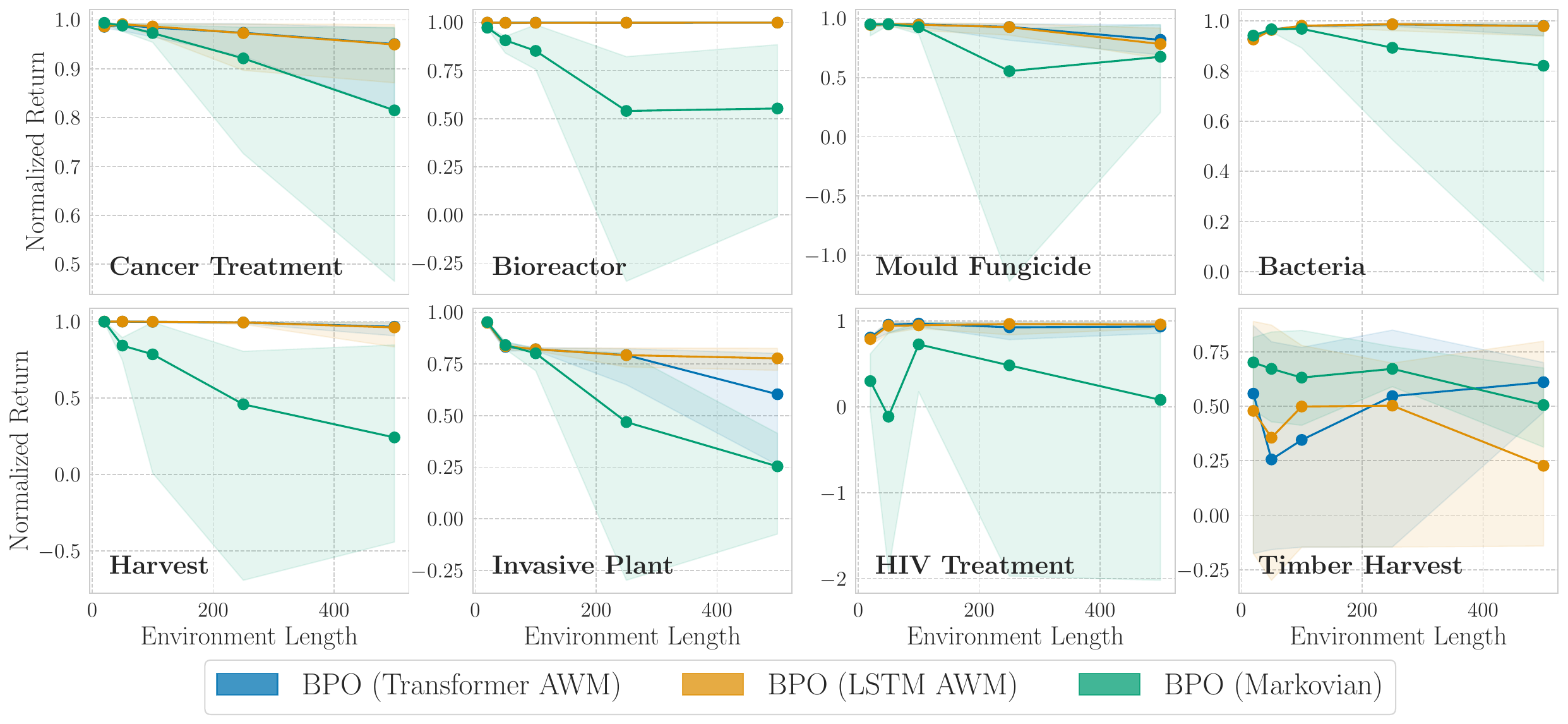}
    \caption{\textbf{Ablation with LSTM Actions World Models on Myriad}. Learning curves of BPO with a transformer AWM, LSTM AWM, and a Markovian world model on eight of the Myriad environments (10 seeds $\pm$ min/max).} 
    \label{fig:lstm-myriad}
\end{figure}

\subsection{Additional Details on Online Decision Transformer}
\label{app:odt}
Recently, a sequence modeling perspective of reinforcement learning has shown promising results on various continuous control tasks in an offline RL setting~\citep{ChenLRLGLASM21,zheng2022online, janner2021sequence}. Although our method differs significantly both conceptually and in the problems they solve (see Appendix \ref{app:seq-rl}), we show experimentally, for completeness, that the online decision transformer~\citep{zheng2022online} performs poorly on the low dimensional Myriad suite. We use the code and hyperparameters provided by~\cite{zheng2022online} in a purely online setup, with a sweep on the number of trajectories gathered per iteration due to the online nature of our problem setup. Importantly, decision transformers (DT) must be conditioned on the \textit{return-to-go} (RTG) to derive desired policies. Prior works~\citep{ChenLRLGLASM21,zheng2022online, janner2021sequence} have shown that DTs are robust to these hyperparameters and can achieve good and sometimes better performance even when the RTG is set to an out-of-distribution return that is impossible to attain.

We show in Figure \ref{fig:odt-agg} that DT performs poorly on the Myriad tasks in an online setting. We summarize some of the takeaways from these experiments:
\begin{itemize}
    \item Decision transformers are still not well suited for a purely online training regime, exhibiting worse performance than simple one-step model-based methods. Moreover, the action-sequence model outperforms the decision transformer in every single environment.
    \item Decision transformers require some expert knowledge in many domains, and simply overshooting the return-to-go can yield sub-optimal and sometimes catastrophic results.
    \item These methods do not scale well with the horizon unlike BPO with Actions World Models. Even on relatively short horizons such as 100 time steps, the online decision transformer's performance quickly drops when compared to its performance for 20 time steps.
\end{itemize}

\begin{figure}[!hbt]
    \centering
    \includegraphics[width=0.6\textwidth]{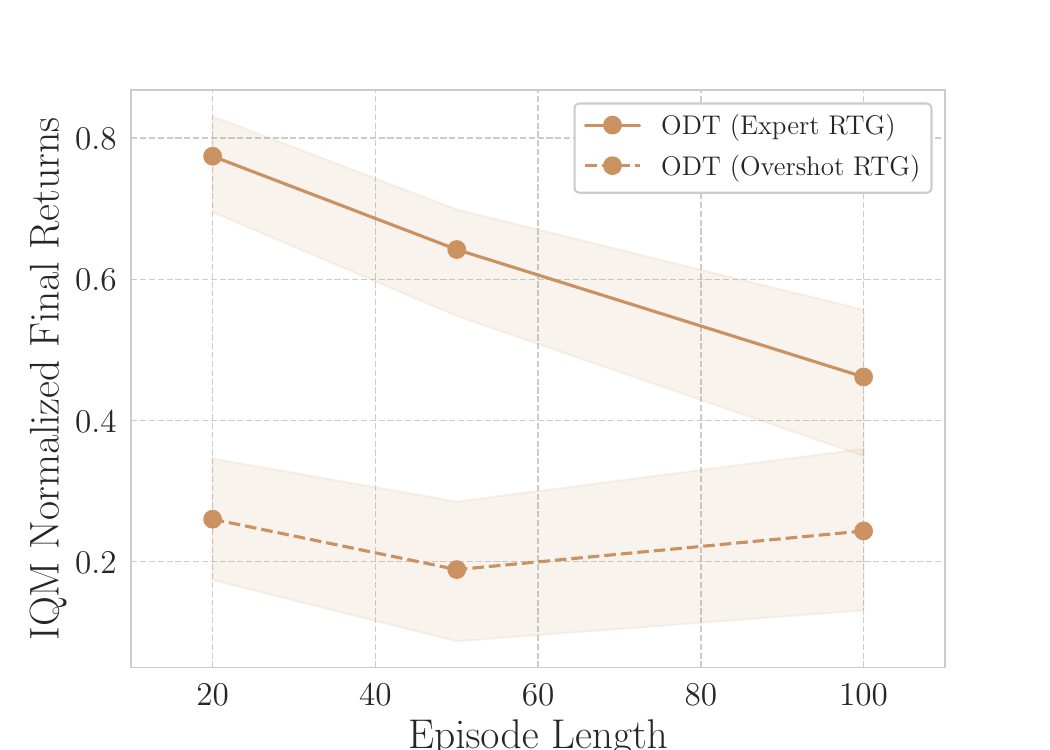}
    \caption{Aggregate final performance of an online decision transformer on the Myriad environments, trained on 200000 environment steps. The \textit{Expert RTG} is conditioned on the optimal return for each environment, while the \textit{Overshot RTG} is conditioned on a return two times higher than the optimal return (10 seeds $\pm$ 95\% C.I.).}
    \label{fig:odt-agg}
\end{figure}

\section{Extended Related Work on Sequence Models in RL}
\label{app:seq-rl}

Sequence models in RL have primarily been used in one of three ways. First, sequence models can be used as history encoders in RL algorithms to maximize returns in partially observable MDPs (POMDPs)~\citep{hausknecht2015deep, ni2021recurrent}, sometimes through a history-dependent world model~\citep{hafner2022mastering, hafner2023mastering}. Second, sequence models have recently shown promise in an imitation learning or offline reinforcement learning setting by treating MDPs as a sequence modeling problem~\citep{ChenLRLGLASM21,zheng2022online, janner2021sequence}, usually conditioning on returns to derive desired policies. Lastly, a separate line of work has used sequence models to reshape the reward landscape for improved temporal credit assignment \citep{hung2019optimizing, arjona2019rudder, liu2019sequence}. In contrast to all of these, our framework is the only one to use an \textit{action-only} conditioned sequence model to directly improve long-term policy gradients in MDPs with no intermediate step. Below, we go into a detailed comparison with each area.

\textbf{Comparison with history-conditioned RL methods for POMDPs.} 
History-conditioned encoders, modeled as sequence models, are often used in POMDPs in both model-free~\citep{hausknecht2015deep, ni2021recurrent} and model-based methods~\citep{hafner2022mastering, hafner2023mastering}. The latter is typically also concerned with predicting observations, they are conditioned on entire history information, while our method includes only actions. The problem setup is also different; such methods are usually concerned with memory in POMDPs~\citep{ni2023transformers}, while ours seeks to improve credit assignment in MDPs.

\textbf{Comparison with decision and trajectory transformers. } Decision transformers \citep{ChenLRLGLASM21,zheng2022online} and trajectory transformers~\citep{janner2021sequence} take a more extreme approach, casting the entire reinforcement learning problem as a sequence modeling one. Conversely, we specifically draw a parallel between policy gradients and sequence models. Their sequence models are conditioned on entire trajectories, which include states, actions, rewards and returns. While trajectory transformers predict states just like our Actions World Models, both trajectory and decision transformers also predict actions. In either case, their sequence models must first be trained on an offline dataset in an imitation learning~\citep{janner2021sequence}, offline RL~\citep{ChenLRLGLASM21}, or pre-training framework~\citep{zheng2022online}, while AWMs work for an online RL setting. 

\textbf{Comparison with reward reshaping methods.} Perhaps more relevant, prior works have already tried harnessing advanced sequence models for improved temporal credit assignment in RL~\citep{hung2019optimizing, arjona2019rudder, liu2019sequence}. In these cases, the predictive power of sequence models, either an LSTM~\citep{arjona2019rudder} or transformer~\citep{hung2019optimizing, liu2019sequence} are used to redistribute or augment the given reward function to produce a surrogate reward function. This surrogate reward is then used in a more traditional model-free RL algorithm. Again, we stress a fundamental difference in the inputs of the sequence models, which all include state information. Our framework also establishes a more direct path between sequence models and credit assignment, avoiding any intermediate steps like reward reshaping. 

\end{document}